\documentclass{article}

\usepackage[utf8]{inputenc} %
\usepackage[T1]{fontenc}    %
\usepackage[colorlinks,citecolor=blue!70!black,linkcolor=blue!70!black,
urlcolor=blue!70!black]{hyperref}       %
\usepackage{url}            %
\usepackage{booktabs}       %
\usepackage{amsfonts}       %
\usepackage{nicefrac}       %
\usepackage{microtype}      %
\usepackage{xcolor}         %
\usepackage{wrapfig}
\usepackage{sidecap}

\usepackage{fullpage}

\usepackage{graphicx}
\usepackage{subfigure}
\usepackage{booktabs} %
\usepackage{standard_definitions}

\usepackage{enumitem}

\usepackage{multirow}
\usepackage{floatrow}
\newfloatcommand{capbtabbox}{table}[][\FBwidth]

\allowdisplaybreaks
\usepackage{natbib}

\usepackage{xspace}

\usepackage{amsthm,amsmath,amssymb}
\usepackage{amsfonts,dsfont}
\usepackage{mathtools}
\usepackage{tikz}
\usepackage{comment} 

\usepackage{caption}

\newcount\Comments  %
\definecolor{darkgreen}{rgb}{0,0.5,0}
\definecolor{darkred}{rgb}{0.7,0,0}
\definecolor{teal}{rgb}{0.3,0.8,0.8}
\newcommand{\kibitz}[2]{\ifnum\Comments=1\textcolor{#1}{\textsf{\footnotesize #2}}\fi}

\renewcommand{\hat}{\widehat}

\newcounter{relctr} %
\everydisplay\expandafter{\the\everydisplay\setcounter{relctr}{0}} %

\newcommand\labelrel[2]{%
  \begingroup
    \refstepcounter{relctr}%
    \stackrel{\textnormal{(\alph{relctr})}}{\mathstrut{#1}}%
    \originallabel{#2}%
  \endgroup
}
\AtBeginDocument{\let\originallabel\label} %

\title{Investigating the Role of Negatives in \\Contrastive Representation Learning}
\author{
  Jordan T. Ash$^*$ \qquad Surbhi Goel$^*$ \qquad Akshay Krishnamurthy$^*$ \qquad Dipendra Misra\footnote{Authors listed in alphabetial order. All authors contributed equally.} \\
  \vspace{-2mm} \\
  Microsoft Research NYC\\
\texttt{\{ash.jordan, goel.surbhi, akshaykr, dimisra\}@microsoft.com} 
}
\date{}

\begin{document}

\maketitle

\begin{abstract}

Noise contrastive learning is a popular technique for unsupervised representation learning. In this approach, a representation is obtained via reduction to supervised learning, where given a notion of semantic similarity, the learner tries to distinguish a similar (positive) example from a collection of random (negative) examples. The success of modern contrastive learning pipelines relies on many parameters such as the choice of data augmentation, the number of negative examples, and the batch size; however, there is limited understanding as to how these parameters interact and affect downstream performance. We focus on disambiguating the role of one of these parameters: the number of negative examples. Theoretically, we show the existence of a collision-coverage trade-off suggesting that the optimal number of negative examples should scale with the number of underlying concepts in the data. Empirically, we scrutinize the role of the number of negatives in both NLP and vision tasks. In the NLP task, we find that the results broadly agree with our theory, while our vision experiments are murkier with performance sometimes even being insensitive to the number of negatives. We discuss plausible explanations for this behavior and suggest future directions to better align theory and practice.

\end{abstract}

\section{Introduction}
\label{sec:intro}

Unsupervised representation learning refers to a suite of algorithmic
approaches for discovering meaningful transformations of complex data in the
absence of an explicit supervision signal. These representations are trained to preserve and expose semantic information in the data, so that once learned, they can be used effectively in downstream supervised tasks of interest. In recent years, these methods have seen tremendous success when combined with deep learning models, obtaining state of the art results on a variety of important language and vision tasks~\citep{smith2005contrastive,chopra2005learning,mikolov2013efficient,oord2018representation,bachman2019learning,clark2020electra,chen2020simple}.

One popular approach for representation learning
is contrastive learning, or noise contrastive estimation (NCE)~\citep{gutmann2010noise}. In
NCE, a notion of semantic similarity is used to create a
\textit{self-supervision} signal that encourages similar points to be
represented (or embedded) in a similar manner. For example, a typical
approach is to train an embedding so that similar points have large
inner product in the embedding space and dis-similar points have small
inner product. Intuitively, representations trained in this manner capture the semantic similarity of the raw data, so we might expect that they are useful in downstream supervised learning tasks. Indeed, NCE is employed in the \textsc{Electra} language model~\citep{clark2020electra} and has also been used effectively in computer vision tasks at ImageNet scale~\citep{chen2020simple}.

While the implementation details vary, a mathematical abstraction for representation learning with NCE is: (1) a single NCE example consists of $k+2$ raw inputs, $(x,x^+,
x_1^-,\ldots,x_k^-)$, where $(x,x^+)$ are semantically similar and
$x_i^-$ are sampled from some base distribution, (2) the
representation $f$ is trained to encourage $f(x)^\top f(x^+) \gg
f(x)^\top f(x_i^-)$ for each $i$. This second step can be done with
standard classification objectives, such as the cross-entropy loss,
where the model is viewed as a classifier over $k+1$ labels. For
example, a candidate objective is
\begin{align*}
\mathop{\EE}_{x,x^+,x_{1:k}^-} \sbr{- \log \rbr{ \frac{e^{f(x)^\top f(x^+)} }{e^{f(x)^\top f(x^+)} + \sum_{i=1}^ke^{f(x)^\top f(x_i^-)}}}}.
\end{align*}
This objective encourages $f(x)^\top f(x^+)$ to be large, so that the representation efficiently captures the semantic similarity inherent in the NCE sampling scheme. 

Modern contrastive pipelines have several parameters that together contribute to the success of resulting representations. These parameters include modeling choices, such as the number of negatives $k$, the semantic sampling procedure for $(x, x^+)$, and neural network architecture, along with optimization decisions, such as batch size. While there have been large scale empirical studies on contrastive learning~\citep{chen2020simple}, it remains unclear how these different choices interact with each other to affect the downstream performance of the learned representations. Indeed, a major challenge is that exorbitant computational costs for training contrastive models precludes comprehensive
ablations to disentangle the effects of each design decision. Instead a hybrid approach that leverages theoretical insights to guide systematic experimentation may be more effective at shedding light on the role of these parameters.

As an example of entangled design choices, consider the number of negative samples $k$ and the batch size $B$ in the NCE objective. In practice, it is standard to set $k$ as a function of the batch size $B$, as $k=2B-2$, and then to set the batch size as large as the hardware can support. When implemented on a GPU, this allows for many more gradient evaluations with minimal computational overhead. However, in this setup, the effect of increasing $k$ on downstream performance is confounded by improvements from increasing the batch size.

In this paper, we focus on the number of negatives $k$, and investigate its role in contrastive learning. The key question we ask is \textit{how does the number of negative samples affect the quality of representations learned by NCE?} As mentioned, existing empirical evidence confounds $k$ with the batch size $B$, so it does not provide much of an answer to this question. Moreover, conceptual reasoning paints a fairly nuanced picture.
On one hand, a large value of $k$ implies that the
negatives are more likely to cover the semantic space (\textit{coverage}), which
encourages the representation to disambiguate all semantically
distinct concepts. However, there is a risk that some of the negative samples will be semantically very similar (\textit{collisions}) to the positive example $x^+$, which compromises the supervision signal and may lead to worse representation quality.

\paragraph{Our contributions.} For our theoretical results, we adopt the framework of~\citet{saunshi2019theoretical}, where the semantic similarity in the contrastive distribution is directly related to the labels in a downstream classification task. We refine the results of~\citet{saunshi2019theoretical} with a novel and tighter \textit{error transfer} theorem that bounds the classification error by a function of the NCE error. We then study how this function depends on the number of concepts/classes in the data and the number of negatives in the NCE task. Our analysis captures the trade-off between coverage over the classes and collisions, and it reveals an expression for the ideal number of negatives $k$, which roughly scales as the number of concepts in the data. We emphasize that our sharper theoretical results are instrumental even for obtaining this qualitative trend; indeed, the analysis in~\citet{saunshi2019theoretical} instead suggests that increasing $k$ leads to decreasing performance.

Empirically, we study the effect of the number of negative samples on the performance of the final classifier on a downstream classification task for both vision and natural language datasets. 
For our language task, we observe that performance increases with $k$ up to a point, but it deteriorates as we increase $k$ further. 
The experiments also confirm that the optimal choice of $k$ depends on the number of concepts in the dataset. 
We do not observe the same trends in our vision experiments, and we describe how conflating factors---most notably data augmentation---cause disagreement between the observed empirical behavior and the predictions from the theoretical framework typically used to study NCE~\citep{saunshi2019theoretical,chuang2020debiased}.\looseness=-1

\paragraph{Related work.}
\label{sec:related}

Self-supervised learning has elicited breakthrough successes in representation learning, achieving state of the art performance across a wide array of important machine learning tasks, including computer vision~\citep{he2020momentum,v1,v3,v4,bachman2019learning}, natural language processing~\citep{n1,n2,n4,n5}, and deep reinforcement learning.
In vision,  \citet{chen2020simple} have shown that simple NCE algorithms are able to obtain extremely competitive performance on the challenging ImageNet dataset, something that was previously thought to require extensive supervised training only.
In NLP, the NCE approach has garnered much attention since the pioneering work of~\citet{mikolov2013efficient} demonstrated its efficacy in learning word-level representations.
In Deep RL, NCE-style losses have been recently employed to learn world models in complex environments~\citep{nachum2018near, srinivas2020curl, rl1, rl2, rl3}.

Theoretical work on  representation learning is quite nascent, with most works focused on mathematical explanations for the efficacy of these approaches~\citep{saunshi2019theoretical,lee2020predicting}. One is via conditional independence and redundancy~\citep{tosh2020contrastivea,tosh2020contrastiveb}. %
Another theory proposed by~\citet{wang2020understanding} is that contrastive representations simultaneously optimize for alignment of similar points and uniformity in the embedding space, in the asymptotic regime where the number of negatives tends to infinity. This asymptotic regime was also studied by~\citet{foster2020unified}, who analyze a contrastive estimator for an information-gain quantity. However,~\citet{foster2020unified} do not study representation learning and~\citet{wang2020understanding} do not theoretically analyze downstream performance as we do. 
Indeed, our results show that there can be a price for too many negatives due to collisions, and perhaps call into question whether this asymptotic regime should be used to explain the downstream performance of NCE representations.\footnote{Note that similar concerns have been raised regarding viewing representation learning as a mutual information maximization problem~\citep{tschannen2019mutual}.} This collision phenomenon has been observed by \citet{chuang2020debiased}, who propose a debiased version of the contrastive loss that attempts to correct for negative samples coming from the positive class. They experimentally observe significant gains from removing the effect of collision. In contrast, here we study the original contrastive objective that is primarily used in practice and characterize how an appropriate choice of $k$ can balance the collision and coverage trade-off.

Our theoretical results closely follow the setup of \citet{saunshi2019theoretical}. We adopt essentially the same framework in the present paper, but unlike~\citet{saunshi2019theoretical} we do not use the number of negatives $k$ to define the downstream classification tasks. Our setup yields a single $k$-independent downstream task, which allows us to study the influence of this parameter. In addition, many steps of our analysis are sharper than theirs, which leads to our new qualitative findings.

\vspace{-0.2cm}
\section{Setup and Algorithm}
\vspace{-0.3cm}
In this section, we present the mathematical framework that we use to study noise contrastive estimation, as well as the learning algorithm and the evaluation metric that we consider.
\label{sec:algo}

\paragraph{Notation.} Let the set of input features be denoted by $\Xcal$ and the set of latent classes by $\mathcal{C}$, where $|\mathcal{C}| = N$. With each latent class $c \in \mathcal{C}$ we will associate a distribution $\mathcal{D}_c$ over $\Xcal$. We can view $\mathcal{D}_c$ as the distribution over data conditioned on belonging to latent class $c$. We will also assume a distribution $\rho$ on $\mathcal{C}$. We will learn with a function class $\Fcal$, the set of representation functions $f: \Xcal \rightarrow \mathbb{R}^d$ that map the input to a $d$-dimensional vector. We use $\lesssim$ to denote less than or equal to up to constants.

\paragraph{NCE data.}
We assume access to {\it similar} data points in the form of pairs $(x, x^+)$ and $k$ {\it negative} data points $x^-_1, \ldots, x^-_k$. To formalize this, an unlabeled sample from $\Dcal_{\unsup}$ is generated as follows:
\begin{itemize}
    \item Choose a class $c$ according to $\rho$ and draw i.i.d. samples $x, x^{+}$ from $\Dcal_c$.
    \item Choose classes $c^{-}_1, \ldots, c^{-}_k$ i.i.d. from $\rho$ and draw $x^{-}_1 \sim \Dcal_{c^{-}_1}, \ldots, x^{-}_k \sim \Dcal_{c^{-}_k}$.
    \item Output tuple $(x, x^{+}, x^{-}_{1:k})$
\end{itemize}
This data generation is implicit. We only observe the data points and not the associated class labels. Observe that $c^{-}_i$ may be the same as $c$ for one or more $i \in \{1, \ldots, k\}$, causing \emph{collisions}. 

Empirically, similar pairs are typically generated from co-occurring data, for example, two parts of a document can be viewed as a similar pair~\citep{tosh2020contrastivea}. Note that our data generation process implies that $x$ and $x^+$ are independent conditioned on the latent class. This is a simplification of the real-world setting where the independence structure might not hold. However, it enables a transparent theoretical analysis that can provide insights that are applicable to some real-world setups. 

\paragraph{NCE algorithm.}
The objective of an NCE algorithm is to use the available unsupervised data to learn a representation that represents (or embeds) similar points in a similar manner and distinguishes them from the negative samples. The following NCE loss captures this objective,
\begin{definition}[NCE Loss]
The NCE loss for a representation $f$ on the distribution $\Dcal_{\unsup}$ is defined as
\begin{align*}
    \Lcal^{(k)}_{\unsup}(f) \defeq \underset{\Dcal_{\unsup}}{\EE}\left[\ell\left(\left\{f(x)^T\left(f(x^+) - f(x^{-}_i)\right)\right\}_{i=1}^k\right)\right].
\end{align*}
The empirical NCE loss with a set $S$ of $M$ samples drawn from $\Dcal_{\unsup}$ is,
\begin{align*}
    &\hat{\Lcal}^{(k)}_{\unsup}(f;S) \defeq \qquad\frac{1}{M} \sum_{(x, x^+, x_{1:k}^-) \in S}\ell\left(\left\{f(x)^T\left(f(x^+) - f(x^{-}_{i})\right)\right\}_{i=1}^k\right).
\end{align*}
\end{definition}
We restrict $\ell$ to be one of the two standard loss functions, hinge loss $\ell(v) = \max\{0, 1 + \max_i\{-v_i\}\}$ and logistic loss $\ell(v) = \log\left( 1 + \sum_{}\exp(-v_i)\right)$ for a vector $v \in \mathbb{R}^k$. Both loss functions and their variants are deployed routinely in practical NCE implementations~\citep{schroff2015facenet, chen2020simple}. 

The algorithm we analyze is the standard empirical risk minimization (ERM) on $\hat{\Lcal}^{(k)}_{\unsup}$ over $\Fcal$. For the rest of the paper, we will denote our learned representation $\hat{f} \in \arg \min_{f \in \Fcal}\hat{\Lcal}^{(k)}_{\unsup}(f)$ %

\paragraph{Downstream supervised learning task.}
For evaluating representations, we will consider the standard supervised learning task of classifying a data point into one of the classes in $\mathcal{C}$. More formally, a sample from $\Dcal_{\super}$ is generated as follows: (1) choose a class $c$ according to $\rho$ (2) draw sample $x$ from $\Dcal_c$, and (3) output sample $(x, c)$.

Observe that the distribution over classes and the conditional distribution over data points is the same as in the unsupervised data. This is crucial to be able to relate the NCE task to the downstream classification task. We will use the corresponding supervised loss for function $g: \Xcal \rightarrow \mathbb{R}^N$,
\begin{align*}
    \Lcal_{\sup}(g) = \underset{\Dcal_{\super}}{\EE}\left[\ell\left(\left\{g(x)_c - g(x)_{c'}\right\}_{c' \ne c}\right)\right]
\end{align*}
Note that~\citet{saunshi2019theoretical} consider a different downstream task which depends on the number of negative samples in the NCE task. Here we disentangle the number of negatives used in NCE from the number of downstream classes which results in a more natural downstream task. %

To evaluate the performance of the learned representation function $\hat{f}$, we will train a linear classifier $W \in \mathbb{R}^{N \times d}$ on top of the learned embedding and use it for the downstream task. We overload notation and define the supervised loss of an embedding $f$ as
\begin{align*}
    \Lcal_{\super}(f) = \inf_{W \in \mathbb{R}^{N \times d}} \Lcal_{\super}(Wf).
\end{align*}
We define the loss of the mean classifier as $\Lcal_{\super}^\mu(f) = \Lcal_{\super}(W^\mu f)$ where $W^\mu$ is such that the $c^{th}$ row is the mean $\mu_c = \EE_{x \sim \Dcal_c}[f(x)]$ of the representations of data from a fixed class $c$. Note that $\Lcal_{\super}(f) \le \Lcal_{\super}^\mu(f)$ for all $f$. Our guarantees will hold for the loss of the mean-classifier which will directly imply a bound on the overall supervised loss. Note that this analysis may be loose if the mean-classifier itself has poor downstream performance\footnote{ \citet{saunshi2019theoretical} present a simple example where the two losses are very different.}. 

\begin{wrapfigure}{r}{0.45\textwidth}
    \centering
    \vspace{-0.8cm}
    \includegraphics[width=0.9\textwidth]{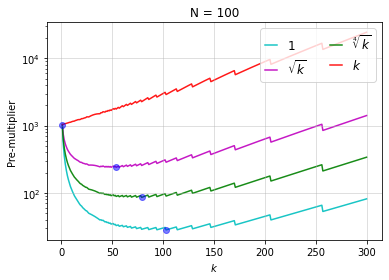}
    \caption{Value of pre-multiplier $\alpha(k,N)$ multiplied with the growth rates of $\Lcal^{(k)}_{\unsup} - \tau_k$ (on a log-scale) for $N=100$ with varying $k$ and growth rates. The plots represent the scaling of our upper bound on the supervised loss. The blue dots represent the optimal $k$ predicted by our theory.}
    \label{fig:varyL}
\end{wrapfigure}

\section{Theoretical Overview}
\label{sec:theory}
In this section, we present our main theoretical insights and overview our analysis, which provides sharp bounds on the supervised loss in terms of the NCE loss. 
Our bounds show an interesting trade-off between the number of negatives and the number of latent classes. For ease of presentation, the bounds in the main paper assume that $\rho$ is uniform over $\mathcal{C}$. 
Proofs and generalizations are deferred to \pref{app:non-uniform}.

In the supervised loss, data points are distinguished against all classes other than the class that the point belongs to. However, for the NCE loss, if the number of negatives $k$ is small relative to $N$, then the negatives may not cover all the classes simultaneously. On the other hand, if $k$ is too large, then it is quite likely that some of the negatives will belong to the same class as the ``anchor point'' $x$, causing collisions that force the learner to distinguish between similar examples. 
This leads to opposing factors when choosing the optimal number of negative samples. We present our main result that captures this trade-off.
\begin{theorem}[Transfer between NCE and supervised Loss]
\label{thm:transfer}
For any $f \in \Fcal$,
\begin{align*}
    &\Lcal_{\super}(f) \le \Lcal^\mu_{\super}(f) \le \underbrace{\frac{2}{1 - \tau_k}\left\lceil \frac{2(N-1) H_{N-1}}{k} \right\rceil}_{ \alpha(k,N)} \left(\Lcal^{(k)}_{\unsup}(f) - \tau_k\right),
\end{align*}
where $\tau_k = \Pr[\exists~i: c^-_i = c] = 1 - \left(1 - \frac{1}{N}\right)^k$ is the probability of a collision amongst the negatives and $H_t=\Theta(\log(t))$ is the $t^{\text{th}}$ harmonic number.
\end{theorem}

Note that the above bound is similar to the one of \citet{saunshi2019theoretical}, except that the left hand side and the pre-multiplier $\alpha(k,N)$ are different. 
The left hand side here is the supervised loss on all classes rather than the loss, in expectation, on $k+1$-way classification task, where the tasks are drawn randomly from $\rho$. 
Further, our pre-multiplier incorporates the trade-off between $k$ and $N$. 
A key difference is that the bound in the prior work gets worse with increasing $k$ implying $k=1$ is the best choice whereas ours has a more subtle dependence on $k$.

\begin{remark}[Choosing $k$]The coefficient $\alpha(k,N)$ decreases with $k$ up to $k \approx \frac{1}{\log(N) - \log(N-1)}$ and then increases.
Note that $\Lcal_{\unsup} - \tau_k$ may itself increase as a function of $k$, however as long as its growth rate is $o(k)$, we will see a similar trend. (See \pref{fig:varyL}).
\end{remark}

We refine our bound by decomposing $\Lcal_{\unsup} - \tau_k$ into two terms that can potentially be small for large $\Fcal$.
\begin{theorem}[Refined NCE Loss Bound]\label{thm:split}
For any $f \in \Fcal$,
\begin{align*}
    \Lcal^\mu_{\super}(f) & \lesssim \alpha(k, N)\left(\Lcal^{(k)}_{\ne}(f) + \sqrt{\frac{k}{N}} s(f)\right)
\end{align*}
where $\Lcal^{(k)}_{\ne}(f)$ is the NCE loss obtained by ignoring the coordinates corresponding to the collisions and $s(f)$ is a measure of average intra-class deviation.
\end{theorem}
We defer the formal definitions of $\Lcal^{(k)}_{\ne}(f)$ and $s(f)$ to \pref{app:split}. Note that the trade-off between $N$ and $k$ still holds here even for the coefficient in front of $s(f)$ (see \pref{fig:varyL} for $\sqrt{k}$). $\alpha(k,N)\sqrt{\frac{k}{N}}$ is minimized roughly at $k \approx \frac{0.5}{\log(N) - \log(N-1)}$. 

Lastly, we do a careful generalization analysis to guarantee closeness of the empirical NCE loss of $\hat{f}$ to the true NCE loss. Given a generalization guarantee, using the fact that $\hat{f}$ is the ERM solution, we obtain the following guarantee:

\begin{theorem}[Guarantee on ERM]
For $\hat{f} \in \arg \min_{f \in \Fcal}\hat{\Lcal}^{(k)}_{\unsup}(f)$ and any $f \in \Fcal$, 
 \begin{align*}
    \Lcal_{\super}(\hat{f}) & \lesssim \alpha(k, N)\left(\Lcal^{(k)}_{\ne}(f) + \sqrt{\frac{k}{N}} s(f) + \sqrt{k}\mathsf{Gen}_M \right)
\end{align*}
where $\mathsf{Gen}_M$ is a decreasing function of $M$ that depends on $\Fcal$ but is independent of $k$.
\end{theorem}
See \pref{app:gen} for a formal definition of $\mathsf{Gen}_M$, which is based on the Rademacher complexity of $\Fcal$.

\subsection*{An example exhibiting the coverage-collision trade-off}
Our bounds suggest that there are two competing forces influenced by the choice of $k$ in the NCE task. The first is coverage of all the classes, which intuitively improves supervised performance, and suggests taking $k$ larger. On the other hand, when $k$ is large, the negatives are more likely to collide with the anchor class $c$, which will degrade performance. However, as the trends regarding the optimal choice of $k$ are based on analyzing an upper bound, 
there is always the concern that the upper bound is loose. To further substantiate our results, we now provide a simple example which demonstrates the existence of the coverage-collision trade-off and also verifies the sharpness of our transfer theorem.

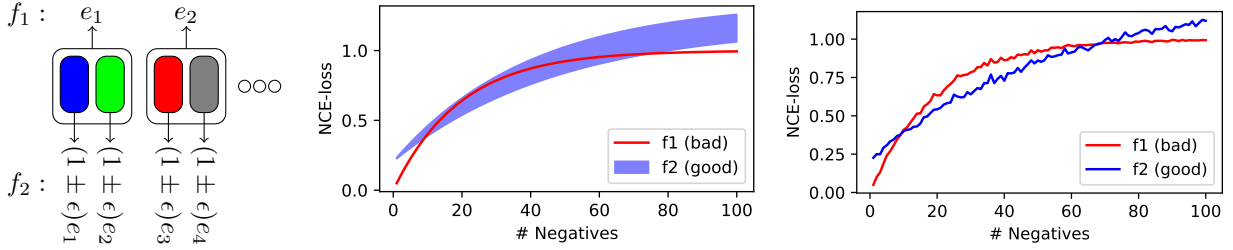
\begin{figure*}
    \centering
    \begin{tikzpicture}
\draw[rounded corners,fill=blue] (-0.15,-0.6) rectangle (0.225,0.15);
\draw[rounded corners,fill=green] (0.325,-0.6) rectangle (0.7,0.15);
\draw[rounded corners,fill=red] (1.1,-0.6) rectangle (1.475,0.15);
\draw[rounded corners,fill=gray] (1.575,-0.6) rectangle (1.95,0.15);
\draw[rounded corners] (-0.25,-0.75) rectangle (0.8, 0.25);
\draw[rounded corners] (1.0,-0.75) rectangle (2.05, 0.25);
\node at (-0.6, 0.7) {$f_1:$};
\draw[->] (0.275,0.25) -- (0.275, 0.55);
\node at (0.275,0.7) {$e_1$};
\draw[->] (1.525,0.25) -- (1.525, 0.55);
\node at (1.525,0.7) {$e_2$};
\node at (-0.6, -1.55) {$f_2:$};
\draw[->] (0.0375,-0.6) -- (0.0375,-1.0);
\draw[->] (0.5125,-0.6) -- (0.5125,-1.0);
\draw[->] (1.2875,-0.6) -- (1.2875,-1.0);
\draw[->] (1.7625,-0.6) -- (1.7625,-1.0);
\node[rotate=270] at (0.0375,-1.7) {$(1\pm\epsilon)e_1$};
\node[rotate=270] at (0.5125,-1.7) {$(1\pm\epsilon)e_2$};
\node[rotate=270] at (1.2875,-1.7) {$(1\pm\epsilon)e_3$};
\node[rotate=270] at (1.7625,-1.7) {$(1\pm\epsilon)e_4$};

\draw (2.3,-.25) circle [radius=0.08];
\draw (2.5,-.25) circle [radius=0.08];
\draw (2.7,-.25) circle [radius=0.08];
\end{tikzpicture}\hspace{0.1cm}
    \includegraphics[width=0.37\textwidth]{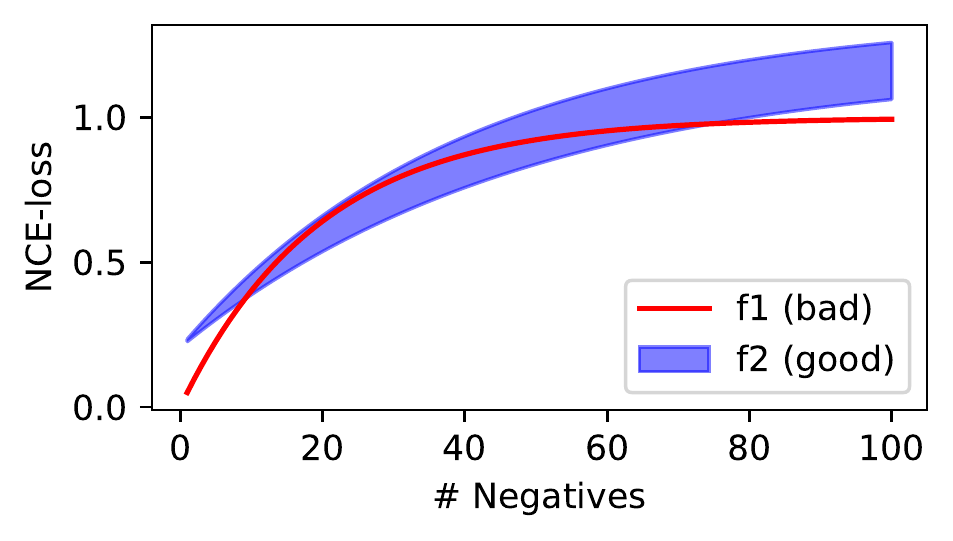}
    \includegraphics[width=0.37\textwidth]{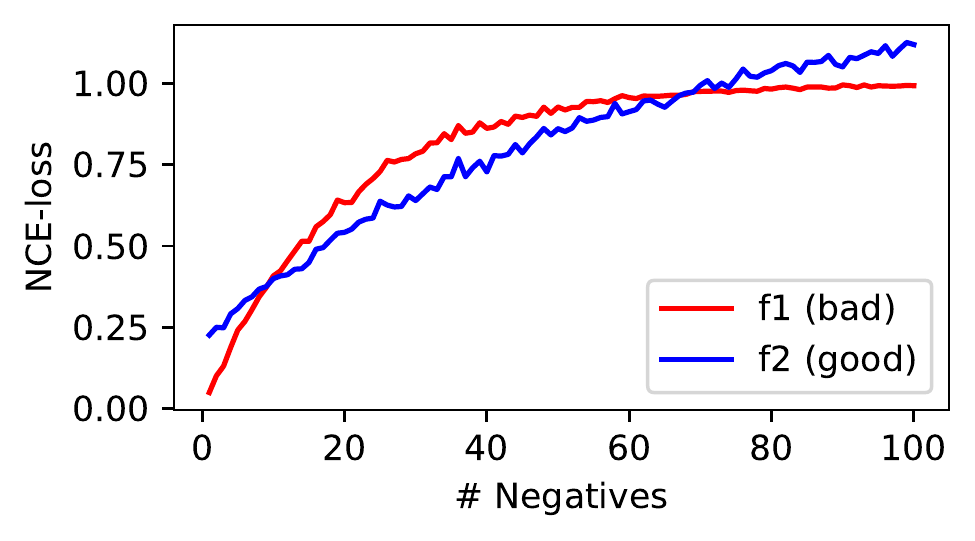}
    \caption{\textbf{Left}: Schematic of the distribution with the two representation functions, where colors correspond to distinct classes. $f_1$ collapses pairs of classes together, while $f_2$ separates all classes, but has some intraclass variance in its embeddings. \textbf{Center}: plot of $\Lcal_{\unsup}^{(k)}(f_1)$ and the bounds on $\Lcal_{\unsup}^{(k)}(f_2)$ as a function of $k$. \textbf{Right}: Monte-carlo simulations of $\Lcal_{\unsup}^{(k)}$ for the two representations, demonstrating both cross points. In both plots we take $N=40,\epsilon=0.35$.}
    \label{fig:example}
\end{figure*}

First we showcase the tightness of~\pref{thm:transfer} and the importance of coverage. Consider a problem with $N$ classes $\{1,\ldots,N\}$ where $N$ is a multiple of $2$ and let $\rho$ be the uniform distribution. For each class $c$, the distribution $D_c$ is the uniform distribution on two examples $x_1^c$ and $x_2^c$. The representation $f_1$ pairs classes together, so that all points from class $2i$ and $2i+1$ are both mapped to the $i^{\textrm{th}}$ standard basis element $e_i$ (see~\pref{fig:example}).

It is straightforward to verify the following facts for the hinge loss (see Appendix for detailed calculations): (1) since $f_1$ collapses classes together, it has $\Lcal_{\sup}^\mu(f_1) = 1$, (2) $\Lcal_{\unsup}(f_1) = 1 - (1-2/N)^k$ since for each class $2i$ the NCE loss will be $1$ if and only if either $2i$ or $2i+1$ appear in the negatives (otherwise the loss will be $0$). Thus, for this function $f_1$ we have
\begin{align*}
    \frac{\Lcal_{\unsup}^{(k)}(f_1) - \tau_k}{1-\tau_k} = 1 - \left(1-\frac{1}{N-1}\right)^k \leq \frac{k}{N-1},
\end{align*}
while $\Lcal_{\sup}^\mu(f_1) = 1$. This verifies that~\pref{thm:transfer} is tight up to constants and the $\Theta(\log(N-1))$ factor arising from the $N-1$-st harmonic number. 

In addition, observe that $\Lcal_{\unsup}^{(k)}(f_1)$ increases rapidly with $k$, since when $k$ is large it is very likely that either the anchor class or its pair will appear in the negatives. This demonstrates the importance of coverage, since when $k$ is large, $f_1$, which has high downstream loss, is unlikely to be selected by the NCE procedure.

To make this precise, consider another function $f_2$ such that $f_2(x_1^i) = (1+\epsilon)e_i$ and $f_2(x_2^i) = (1-\epsilon)e_i$. This representation correctly separates all of the classes, but there is some noise in the embedding space. By direct calculation, we have $\Lcal_{\sup}^\mu(f_2) = \epsilon/2$ so that $f_2$ should be preferred for downstream performance. We can also show that
\begin{align*}
    0 \leq \Lcal_{\unsup}^{(k)}(f_2) - (1-\tau_k)\left(\frac{\epsilon^2}{4}+\frac{\epsilon}{2}\right)-\tau_k \leq \tau_k\epsilon.
\end{align*}
See~\pref{fig:example}, where we plot a sharper lower bound that we establish in the appendix. For moderately small $\epsilon$, say $\epsilon = 0.1$ these bounds imply that $\Lcal_{\unsup}^{(k)}(f_1) < \Lcal_{\unsup}^{(k)}(f_2)$ when $k/N$ is sufficiently small ($k\lessapprox 0.05N$), which means that the bad representation $f_1$ will be selected.  On the other hand when $k \gtrapprox 0.05N$, the better representation $f_2$ will be selected. This demonstrates how using many negatives can yield representations with better downstream performance. %

In the asymptotic regime where $k \to \infty$, we observe a second important phenomena with this example: If $k/N$ is too large, then the bad classifier $f_1$ will be favored again. 
To see this, note that when there are many collisions $f_2$ will incur $1+\epsilon$ NCE error, so $\lim_{k \to \infty} \Lcal_{\unsup}^{(k)}(f_2) = 1+\epsilon$. On the other, it is clear that $\lim_{k \to \infty} \Lcal_{\unsup}^{(k)}(f_1) = 1$, so $f_1$ will be selected. This phenomena can be seen in~\pref{fig:example}. In this regime, the collisions in the NCE problem compromise the supervision signal and lead to the selection of a suboptimal representation.

\section{Experiments}
\label{sec:experiments}

We test our theoretical findings on two domains, an NLP task and a vision task. We first describe our general learning setting and then describe these tasks and results.

\paragraph{Experimental setting.}
We train the encoder model $f$ by performing NCE with logistic loss. 
Following~\citet{saunshi2019theoretical},
we train on the NCE objective for a fixed number of epochs and do not use a held-out set for model selection. 
We train a linear classifier on the downstream task for a fixed number of epochs and save the model after each epoch. We do not fine-tune $f$ during downstream training. %

Computing $k$ negative tokens for a batch size $B$ results in $\Ocal(Bk)$ forward passes through the model, which is prohibitively expensive for large values of $B$ and $k$. In order to reduce computational overhead, we follow the data reuse trick used in literature~\citep{chen2020simple,chuang2020debiased}. Given $B$ pairs of positively aligned, semantically similar points $\{(x_i, x^+_i)\}_{i=1}^B$, we compute $\{f(x_i), f(x^+_i)\}_{i=1}^B$ using $2B$ forward passes. For each $i \in [B]$, we treat the set $\Ccal_i = \{x_j, x^+_j \mid j \in [B], j \ne i \}$ as the candidate set of negative samples to pick from for the positive pair $(x_i, x^+_i)$. Previous work use the entire candidate set $\Ccal_i$, which gives $k = |\Ccal_i| = 2(B-1)$. While this allows most negative samples to be used, it results in a confounding affect, due to the entanglement of $B$ and $k$.  Instead, for a fixed value of $k \le 2(B-1)$, we sample $k$ negative examples without replacement from $\Ccal_i$ independently for each $i$. This sampling protocol marks a notable departure from prior work by \emph{decoupling the dependence between batch size and number of negative examples} and studying their influence on performance separately.

\begin{figure}[t]
\centering
\includegraphics[scale=0.43]{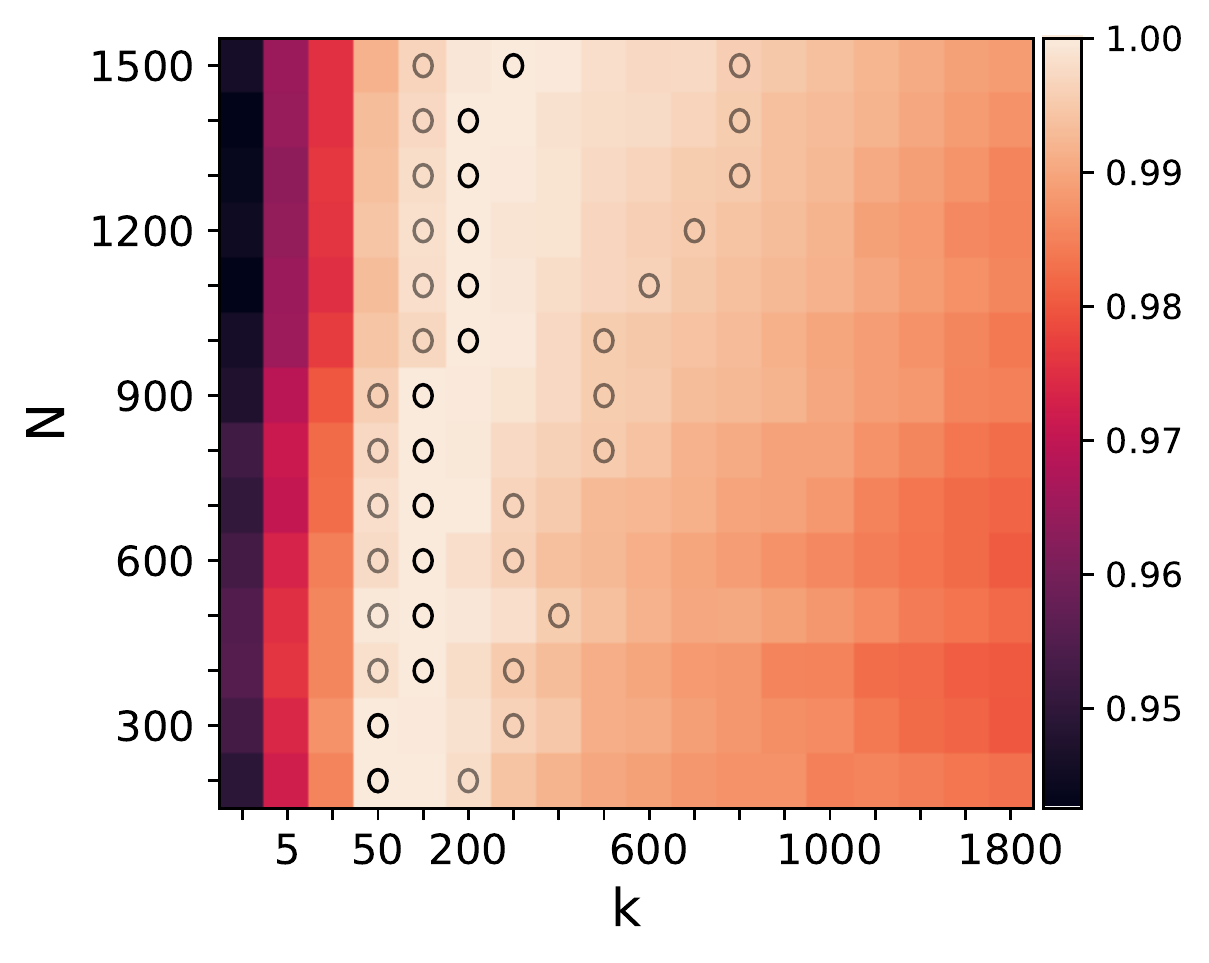}\hspace{-0.2cm}
\includegraphics[scale=0.43]{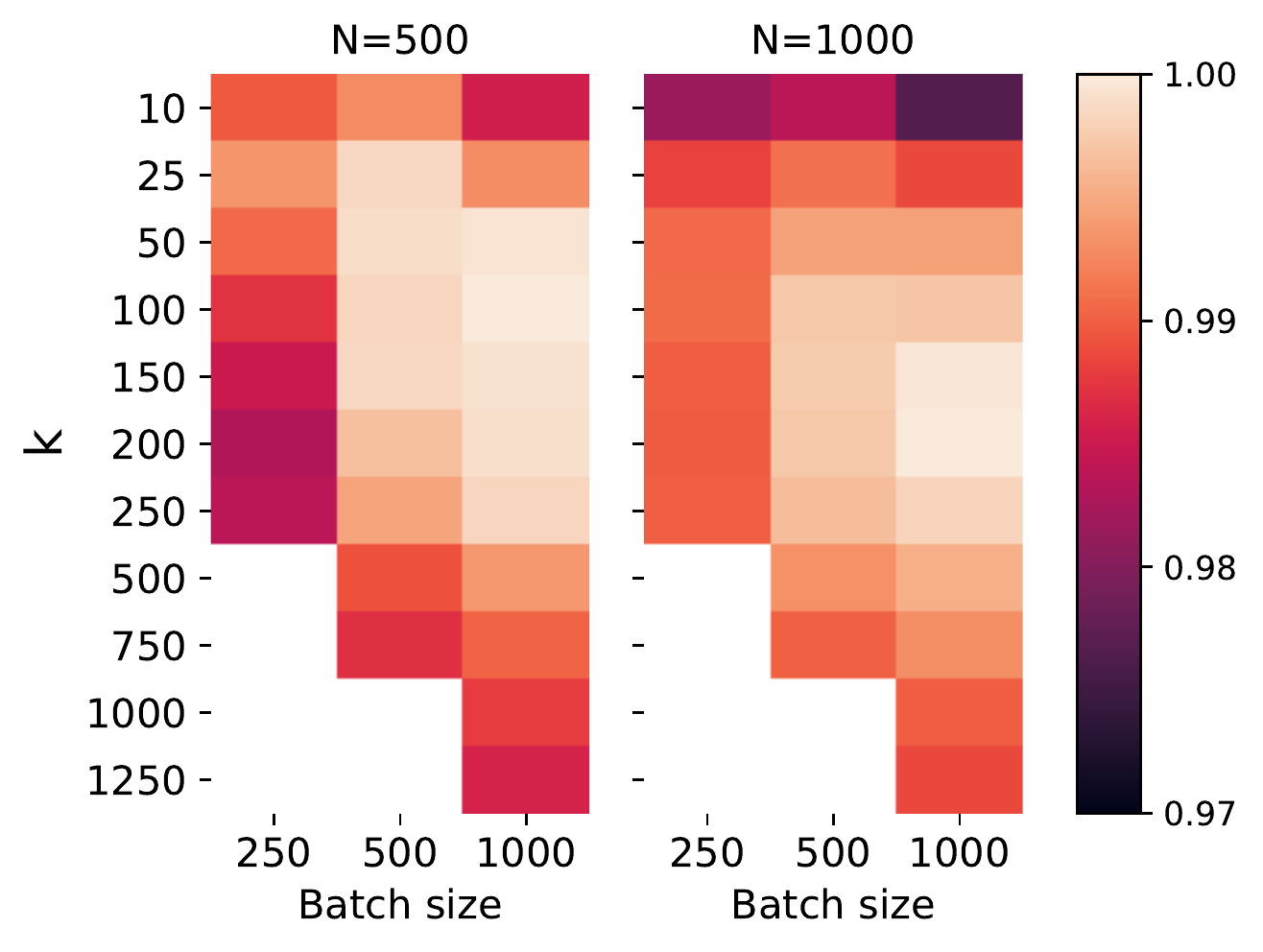}\hspace{-0.2cm}
\includegraphics[scale=0.43]{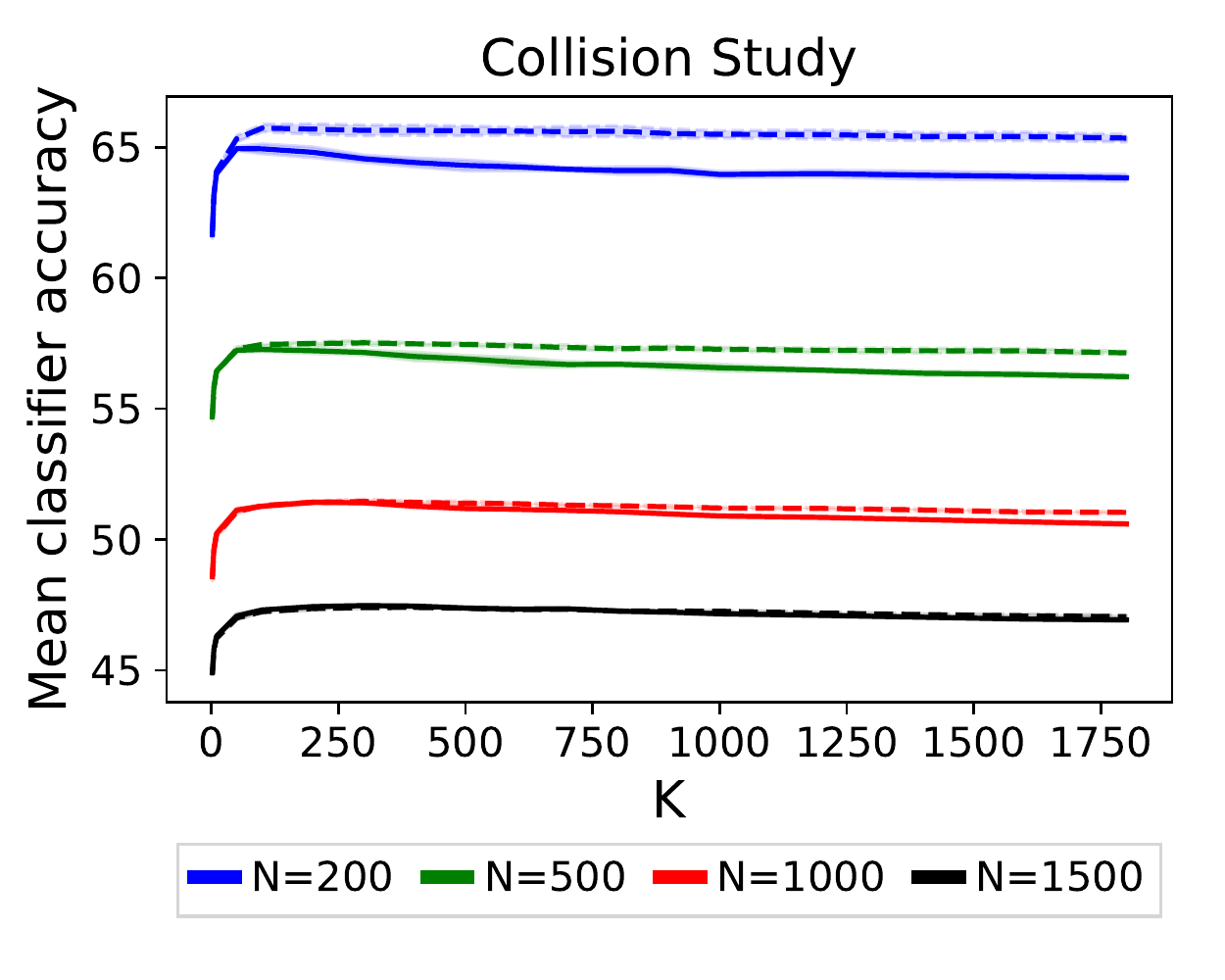}
\caption{Mean classifier accuracy on the \textbf{Wiki-3029} dataset. \textbf{Left:} Downstream performance as a function of $N$ and $k$ for a fixed batch size of $B=1000$. Each row is locally normalized by dividing the unnormalized accuracy in a given row with the best performance in that row. Color denotes normalized average accuracy across 5 trials where lighter areas have higher performance. Solid black circles highlight $k$ with the best performance for each row. Shaded black circles show the smallest and largest value of $k$ with normalized performance of at least $0.995$. 
\textbf{Center:} Mean classifier average accuracy across 5 trials with varying $B$ and $K$. The performance is globally normalized by dividing the unnormalized accuracy by the overall best accuracy. 
\textbf{Right:} An experiment highlighting the effect of disabling collisions across $k$ for 4 values of $N$. Solid lines indicate vanilla training, and dotted lines indicate no collision.}
\label{fig:nlp-exp}
\end{figure}

\paragraph{NLP experiments.} Here we perform a thorough analysis on the Wiki-3029 dataset of~\cite{saunshi2019theoretical}. This dataset consists of 200 sentences taken from 3029 Wikipedia articles. The downstream task consists of predicting the article identity from a given sentence. 
A key advantage of this dataset is that it contains many classes, which allows us to study learning behavior as the number of classes $N$ increases. 
For a given $N \le 3029$, for both training and test datasets, we remove sentences from all but the first $N$ articles. 
We tokenize sentences using the NLTK English word tokenizer~\citep{BirdKleinLoper09}, and replace all tokens that occur less than 50 times in the training set, with token ${\tt UNK}$ denoting \emph{an unknown token}. 
This gives us a vocabulary size $V$ of 15,807 with at most 136 unique tokens per sentence.
We create our training set by randomly sampling 80\% of the sentences from each article. The remaining 20\% of the sentences for each article form the test set.

We use a bag-of-words representation for modeling $f$. Formally, we map a sentence $x$ to a list of tuples $\{(t_i, p_i)\}$ where $\{t_i\}$ is the set of unique tokens in $x$ and $p_i$ is the number of times $t_i$ occurs in $x$ normalized by the number of tokens in $x$. We compute $f(x) = \sum_{i} p_i \nu_{t_i}$ where $\nu \in \RR^{V \times d}$ is a word-embedding matrix with embedding dimension $d$, and $\nu_{t}$ is the $t^{th}$ row denoting the word embedding of token $t$. We initialize the word embedding matrix randomly and train it using NCE. %

We train the NCE model for a fixed number of epochs and do not use a held-out set, following the protocol of~\cite{saunshi2019theoretical}. We measure both the downstream performance of the mean classifier $W^\mu$ and a trained linear classifier. For the former, we use the entire training dataset to compute $W^\mu$ using the trained representation. For the latter, we use the entire training dataset, with labels, to fit a linear layer on top of the model $f$, and we do not fine-tune $f$. For each setting, we repeat the experiment 5 times with different seeds. Please see~\pref{app:experiment-details} for full details on hyperparameters values and optimization details.

The leftmost plot in \pref{fig:nlp-exp} shows the accuracy of the mean classifier as we vary $N$ and $k$. 
For a given number of classes $N$, performance improves as $k$ increases and then slowly decays for larger values of $k$. 
We show the general trend in this figure and report exact numbers to the appendix. 
Further, we find that the optimal choice of $k$, indicated by solid black circles, slowly increases as $N$ increases. 
Both of these empirical findings support the theoretical predictions of our improved error transfer bounds. 
As we have discussed, the previous theoretical results~\citep{saunshi2019theoretical} do not predict this observed trend.
Lastly, we observe that performance improves very rapidly as $k$ increases, up until $k=50$, 
while the decay in performance as $k$ increases beyond the optimal value is relatively mild in comparison to these initial gains.\looseness=-1

The center plot in~\pref{fig:nlp-exp} shows the mean classifier's performance as we vary $B$ and $k$ for two values of $N$. 
Note that for a given value of $B$, we can only increase $k$ up to $2(B-1)$. 
We observe that increasing the batch size can improve the performance even when the number of negative examples are fixed. 
Optimal performance is obtained for the largest batch size and an intermediate value of $k$. 
This highlights the very different roles played by $B$ and $k$ on this dataset. 
We note that these results could not have been observed by previous work that tie the value of $k$ to $B$.

\begin{figure}[t]
\hspace{-0.8cm}
\includegraphics[trim={0cm, 0cm, 0cm, 0}, clip, scale=0.58]{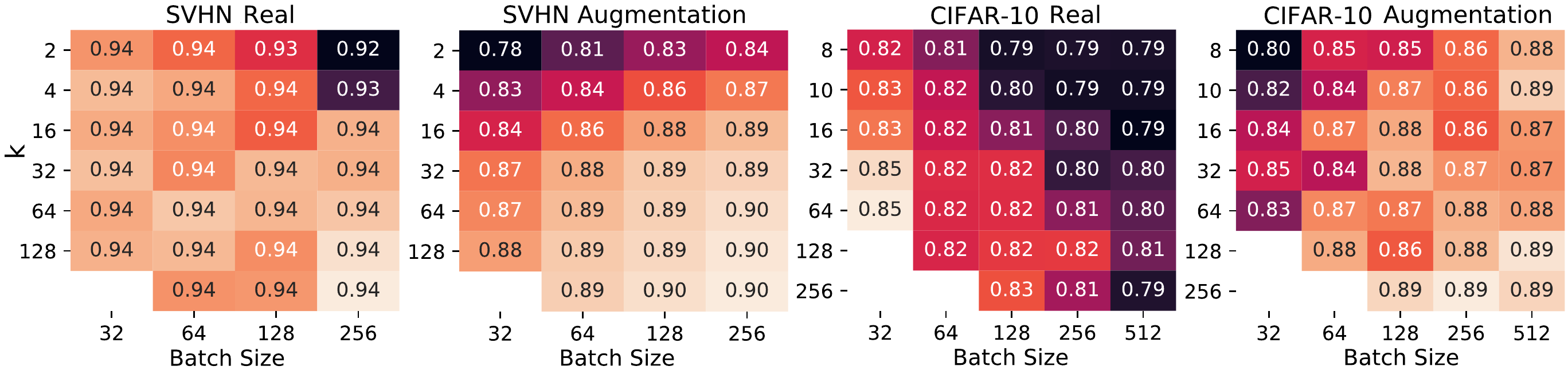}
\caption{Average downstream classifier accuracy as a function of both batch size and $k$ for various vision datasets and NCE parameters. \textbf{Left}: SVHN data, using a true positive sample from the same class, as used in our analysis. \textbf{Center Left}: SVHN data, but using data augmentation. Here the positive sample is an augmented version of the reference sample, and negatives are augmented as well. \textbf{Center Right}: NCE on CIFAR-10 data, using real positive examples. \textbf{Right:} CIFAR-10 data with data augmentation.}
\label{fig:visionExp}
\end{figure}

Our theoretical analysis reveals a coverage-collision trade-off and suggests that the poor performance when $k$ is very large is due to many collision between the negative samples and the anchor/positive, i.e., \emph{false negatives}.
To empirically verify this claim, we modify the representation learning procedure by sampling
$k$ negative examples only from those members in the candidate set that do not share the same class as the positive reference point.\footnote{If we are unable to find $k$ true negatives for any example in the batch, then no update is performed on this batch. However, due to the choice of $N, B, k$ this never occurred in our experiments.} 
In the rightmost plot of~\pref{fig:nlp-exp}, we compare this sampling procedure with the standard one for different values of $N$ and $k$ and a fixed $B=1000$. 
We observe that artificially removing collisions improves the performance across all values of $N$; however, the effect is more significant for smaller values of $N$ ($1.5\%$ for $N=200$ and $K=1000$). 
That the effect is more significant for small $N$ is to be expected, since the probability of collision decreases rapidly with $N$, so the different sampling procedures are quite similar for large $N$.
We can further notice that performance decays only slightly for very large values of $k$ in the absence of collisions. 
This evidence supports the collision hypothesis, namely that collisions are a central reason for performance degradation with increasing $k$. 
This observation was also made by~\citet{chuang2020debiased}. 
Lastly, as hypothesized by prior work~\citep{chen2020simple}, we believe a mild reduction in accuracy with large values of $k$ can happen even in the absence of collision, due to instability in optimization

\paragraph{Vision experiments.}

In this section, we decouple batch size from $k$ and measure downstream performance as a function of these parameters on the SVHN and CIFAR-10 datasets. 
We fit all parameters except the last layer of a ResNet-18 model using the NCE loss, and we evaluate downstream accuracy by training the last layer only in a fully supervised fashion. 
We use a learning rate of $10^{-3}$, train for 400 epochs, and fit parameters with the Adam optimizer.\looseness=-1

We consider two ways of sampling tuples of images for the NCE objective. 
In the first, we adhere to our theoretical framework and use class information to capture semantic similarity. Here, for a given sample, we use an image from the same class to act as a positive, and we do not use data augmentation. 
In the second, we deviate from our theory and do not use class information. Instead, we follow a common empirical approach and use an augmented version of the anchor point. 
Negative samples are also augmented. 
For these experiments, we exercise the SimCLR data augmentation scheme, consisting of randomly applied crops, color jitter, horizontal flips, greyscaling, and blurring~\citep{chen2020simple}.

We experiment with both SVHN and CIFAR-10 image datasets, with \pref{fig:visionExp} summarizing our results. 
When using augmentation, the trend shown in previous work is observed, with accuracy tending to improve with increasing batch size and $k$. 
When using class information for sampling positives, however, the performance trends are somewhat unexpected. %
Downstream SVHN accuracy in this setting is remarkably consistent, with nearly all combinations of $k$ and batch size producing roughly equivalent, high-quality representations for downstream classification. 
In the CIFAR-10 case, this is not true---in fact, performance tends to increase as batch size decreases. This is likely due to the exclusion of data augmentation, which is necessary for obtaining reasonable performance even in the fully supervised regime for these architectures. In its absence, the variance induced by decreasing the batch size appears to offer a regularization benefit that improves accuracy. Still, overall performance without augmentation is poor in comparison to data-augmented NCE.

These experiments raise an interesting conjecture: perhaps the observation that downstream performance increases with $k$ and batch size can be attributed to the kinds of data augmentation techniques that have been refined specifically for vision, rather than to the NCE objective itself. 
In NLP, where there are less clear-cut ways of performing data augmentation, we see trends that more closely agree with our theory. 
Either way, these experiments point to a distortion between the theoretical framework typically used to study NCE (using true positives) and, in the vision case, what is observed in practice. 
The core of this distortion appears to involve the use of augmentation techniques that have been highly tuned for object recognition/computer vision tasks.

\section{Discussion}
\label{sec:discussion}
In this paper, we study how the number of negative samples $k$ used in contrastive representation learning affects the quality of learned representations. 
Our theoretical results highlight a tension between covering the concepts in the data and avoiding collisions, and suggest that the number of negatives should scale roughly linearly with the number of concepts. Our NLP experiments confirm this trade-off whereas our vision experiments show a more entangled dependence on other factors such as inductive bias of the architecture and the use of data augmentation. 

This article therefore raises a number of interesting directions for future work, spanning both theory and practice, and calls for a more detailed investigation to better understand the success of NCE in different domains. We close with a few of these directions and related points of discussion:

\paragraph{Data augmentation.} As alluded to by our vision experiments, the use of data augmentation in practice seems to be a principal factor in the gap between what is observed empirically and what we describe theoretically. Because NCE is so commonly used with image data, developing a theoretical model that can incorporate properties of commonly used augmentations and the inductive bias of networks is an important avenue for future research.

\paragraph{Choosing $k$ in practice.} 
When the NCE distribution and the downstream task are closely linked, our theoretical results indicate that the number of negatives used in NCE should scale with the number of classes. In practice, however, choosing $k$ is less clear, with the optimal choice being a function of complex factors. The model's inductive bias, the choice of data augmentation used, and the data type all seem to influence the ideal value. Given this, are there algorithmic principles for setting $k$ in practice?

\paragraph{Decoupling negatives with other parameters.} This work disentangles $k$ from the batch size used for model training, with our experiments suggesting that each plays a different role in the quality of the learned representation. Modern NCE pipelines have several conflating parameters, and we believe that a deeper empirical study is imperative for disambiguating and understanding the role of each.

\subsection*{Acknowledgements}
We thank Nikunj Saunshi for helpful discussions regarding the experimental aspects of this work. We also thank Remi Tachet des Combes and Philip Bachman for interesting discussions and for providing references to relevant work.

\bibliography{refs}
\bibliographystyle{abbrvnat}

\newpage

\newpage
\appendix
\section*{Appendix}

The appendix is organized as follows:~\pref{app:useful-facts}-\ref{app:calculations-section-4.1}
present proofs of our theoretical results.~\pref{app:experiment-details} presents experimental details including additional experiments and hyperparameter values.

\section{Useful Facts}
\label{app:useful-facts}
 The standard multi-class hinge loss $\ell(v) = \max\{0, 1 + \max_i\{-v_i\}\}$ and the logistic loss $\ell(v) = \log \left( 1 + \sum_i \exp(-v_i)\right)$, for vector $v \in \mathbb{R}^k$, satisfy the following lemma:
\begin{lemma}[Sub-additivity of losses] \label{lem:loss}
Let $v \in \RR^k$ be any vector. For all $I_1, I_2 \subset [k]$, with $S = I_1 \cup I_2$, we have that
\[
\ell(\{v_i\}_{i \in I_1}) \le \ell(\{v_i\}_{i \in S}) \le \ell(\{v_i\}_{i \in I_1}) + \ell(\{v_i\}_{i \in I_2}).
\]
\end{lemma}

\begin{lemma}[Coupon Collector]\label{lem:coupon}
Consider the coupon collector problem with $n$ coupons, each with probability $p_i$. In each trial we collect a coupon by sampling from the distribution $(p_1,\ldots,p_n)$. Let $M$ denote the minimum number of trials until all coupons have been collected. Then:
\begin{align*}
\EE[M] = \int_0^{\infty}\left(1 - \prod_{i=1}^n(1 - \exp(-p_it)\right) dt \le \frac{H_n}{\min_i p_i},
\end{align*}
where $H_n$ is the $n$-th harmonic number. Furthermore, using Markov's inequality, 
\begin{align*}
    \Pr\sbr{ M \geq \frac{2H_n}{\min_{i} p_i}} \leq \EE\sbr{M} \cdot \frac{\min_{i}p_i}{2H_n} \leq 1/2.
\end{align*}
\end{lemma}

\begin{lemma}\label{lem:losses}
For any vectors $v, w \in \mathbb{R}^t$,
\begin{align*}
\ell(v) - \ell(w) \le \left|\max_i (w_i -v_i)\right|.
\end{align*}
\end{lemma}
\begin{proof}
For hinge loss,
\begin{align*}
\ell(v) &= \max\left(0, 1 + \max_i ((w_i - v_i) - w_i)\right)\\
&\le \max\left(0, 1 + \max_i (-w_i) + \max_i(w_i - v_i)\right)\\ 
&\le \max\left(0, 1 + \max_i (-w_i)\right) + \left|\max_i(w_i - v_i)\right| = \ell(w) + \left|\max_i(w_i - v_i)\right|.
\end{align*}
For logistic loss,
\begin{align*}
\ell(v)  &= \log\left(1 + \sum_i \exp(-v_i)\right)\\
&=  \log\left(1 + \sum_i \exp(-w_i)\exp(w_i - v_i)\right)\\
&\le \log\left(1 + \exp\left(|\max_i(w_i - v_i)|\right)\sum_i \exp(-w_i)\right)\\
&\le |\max_i(w_i - v_i)| + \log\left(1 + \sum_i \exp(-w_i)\right) = \ell(w) + \left|\max_i(w_i - v_i)\right|. 
\end{align*}
The last inequality holds because $1 \leq \exp(|\max_i (w_i - v_i)|)$. 
\end{proof}

\section{General Results with Non-uniform Probabilities and Proofs}
\label{app:non-uniform}

In this section, we state and prove a more general version of~\pref{thm:transfer} and~\pref{thm:split} which can accommodate a non-uniform class distribution $\rho$. The particular results in the paper can be obtained by setting all class probabilities as $1/N$.

\paragraph{Notation.} For a drawn sample, denote the collisions by $I(c, c^-_1, \ldots, c^-_{k}) = \{i \in [k] | c = c^-_i \}$. Let $\tau_k(c) = \Pr_{c^-_i\sim \rho^k}[I(c, c^-_1, \ldots, c^-_{k}) \ne \emptyset] = 1 - (1 - \rho(c))^k$ denote the probability of seeing class $c$ in the negative samples (i.e., the collision probability). Also, let $\tau_k = \sum_{c' \in \Ccal}\rho(c')\tau_k(c') = 1 - \sum_{c' \in \Ccal}\rho(c')(1-\rho(c'))^k$. 
We will drop the arguments of $I$ when it is clear from the context.

\subsection{General Transfer Theorem}
Let us first start with the general transfer theorem.
\begin{theorem}[General Transfer Bound] For all $f \in \Fcal$, we have
\begin{align*}
    \Lcal^{\super}(f) \le \frac{2\max\left(1, \frac{ 2(1 - \rho_{\min})H_{N-1}}{k\rho_{\min}}\right)}{\left(1 - \rho_{\max}\right)^k}\left(\Lcal^{(k)}_{\unsup}(f) - \tau_k \underset{c, c^-_i \sim \rho^{k+1}}{\EE} \left[\ell_{|I|}(0) \middle| I \ne \emptyset\right]\right),
\end{align*}
where $\rho_{\min} = \min_c \rho(c)$, $\rho_{\max} = \max_c \rho(c)$ and $H_t$ is the $t^{\text{th}}$ harmonic number.
\end{theorem}
\begin{proof}
The first step is to apply Jensen's inequality to get,
\begin{align*}
    \Lcal^{(k)}_{\unsup}(f) &= \underset{c, c_i^{-} \sim \rho^{k+1}}{\EE}\underset{\substack{x, x^{+} \sim D_{c}^2\\x^{-}_i \sim D_{c^{-}_i}}}{\EE}\left[ \ell\left(\left\{f(x)^\top \left(f(x^{+}) - f(x^{-}_i\right)\right\}\right) \right]\\
    &\geq \underset{\substack{c,c^-_i \sim \rho^{k+1}\\ x \sim D_{c}}}{\EE}\left[ \ell\left(\left\{f(x)^\top \left(\mu_{c} - \mu_{c^-_i}\right)\right\}\right) \right].
\end{align*}
Introducing the collision function $I$, we have
\begin{align*}
\Lcal^{(k)}_{\unsup}(f)&\ge \underset{x \sim D}{\EE}\underset{c^-_i \sim \rho^k}{\EE}\left[ \ell\left(\left\{f(x)^\top \left(\mu_{c} - \mu_{c^-_i}\right)\right\}\right) \right]\\
&\ge \underset{(x,c) \sim D}{\EE}\left[(1 - \tau_k(c))\underset{c^-_i \sim \rho^k}{\EE} \left[ \ell\left(\left\{f(x)^\top \left(\mu_{c} - \mu_{c^-_i} \right)\right\}\right) \middle| I = \phi, c \right]\right] +  \tau_k\underset{c, c^-_i \sim \rho^{k+1}}{\EE} [\ell_{|I|}(0)| I \ne \emptyset]
\end{align*}
Here $\ell_t(0)$ is equal to $\ell$ evaluated at a $t$-dimensional 0 vector, which we obtain $\ell_{|I|}(0)$ by using the sub-additivity property in~\pref{lem:loss}. Note that for the second term, conditional on $I$, the distribution of $c$ is not $\rho$ in general. We deal with the second term in the next section, and focus now on the first term.
\begin{lemma}
For any $c \in \Ccal$ and $x$, we have,
\begin{align*}
\underset{c^-_i \sim \rho^k}{\EE} \left[ \ell\left(\{f(x)^\top \left(\mu_{c} - \mu_{c^-_i} \right)\}\right) \middle| I = \emptyset, c \right] \ge \frac{1}{2\left\lceil\frac{ 2(1 - \rho(c))H_{N-1}}{k\min_{c' \ne c}\rho(c')}\right\rceil}\cdot \ell\left(\{f(x)^\top \left(\mu_{c} - \mu_{c'} \right)\}_{c' \in \Ccal \setminus c}\right),
\end{align*}
where $H_N$ is the $N^{\textrm{th}}$ Harmonic number. 
\end{lemma}
\begin{proof}
We can view the conditioning on $I = \emptyset$ and $c$ as selecting $c^{-}_i$ from $\rho_{-c}$ with support on all $c' \ne c$ and $\rho_{-c}(c') = \frac{\rho(c')}{1 - \rho(c)}$. Therefore, we have
\begin{align*}
\underset{c^-_i \sim \rho^k}{\EE}&\left[ \ell\left(\{f(x)^\top \left(\mu_{c} - \mu_{c^-_i} \right)\}\right) \middle| I = \emptyset, c \right] \\
&\labelrel={in:neg} \underset{c^-_i \sim \rho_{-c}^k}{\EE} \left[ \ell\left(\{f(x)^\top \left(\mu_{c} - \mu_{c^-_i} \right)\}\right) \right] \\
&\labelrel={in:mcopies} \frac{1}{m}\sum_{j=1}^m\underset{c^-_{ij} \sim \rho_{-c}^k}{\EE} \left[ \ell\left(\{f(x)^\top \left(\mu_{c} - \mu_{c^-_{ij}} \right)\}\right) \right]\\
&= \frac{1}{m}\underset{c^-_{ij} \sim \rho_{-c}^{k \times m}}{\EE} \left[ \sum_{j=1}^m \ell\left(\{f(x)^\top \left(\mu_{c} - \mu_{c^-_{ij}} \right)\}\right) \right]\\
& \labelrel\ge{in:sub-add} \frac{1}{m}\underset{c^-_{ij} \sim \rho_{-c}^{k \times m}}{\EE} \left[\ell\left(\{f(x)^\top \left(\mu_{c} - \mu_{c^\prime} \right)\}_{c^\prime \in \cup \{c^{-}_{ij}\}}\right) \right]\\
&\labelrel\ge{in:pos} \frac{1}{m}\underset{c^-_{ij} \sim \rho_{-c}^{k \times m}}{\EE} \left[ \one\left[\cup \{c^{-}_{ij}\} = \Ccal \setminus \{c\}\right]\ell\left(\{f(x)^\top \left(\mu_{c} - \mu_{c^\prime} \right)\}_{c^\prime \in \Ccal \setminus c}\right) \right]\\
&= \frac{1}{m}\cdot \ell\left(\{f(x)^\top \left(\mu_{c} - \mu_{c^\prime} \right)\}_{c^\prime \in \Ccal \setminus c}\right) \cdot \underset{c^-_{ij} \sim \rho_{-c}^{k \times m}}{\Pr} \left[\cup \{c^{-}_{ij}\} = \Ccal \setminus \{c\}\right].
\end{align*}
We replace the distribution over negatives, as described above to get \eqref{in:neg}. Then we construct $m$ copies of this expression and take an average to get \eqref{in:mcopies}, which is a key step in the argument. Next, \eqref{in:sub-add} uses~\pref{lem:loss}, which allows us to replace the sum over the $m$ losses with a single multi-class loss on the union of the labels. Inequality \eqref{in:pos} uses the non-negativity of the loss and finally we observe that the loss expression no longer depends on the samples $c_{ij}^-$. 

To bound the probability term, note that it can be viewed as an instance of the standard coupon collector problem with $N- 1$ different coupons with non-identical probabilities of selection given by $\rho_{-c}$. In this instance, we perform $mk$ trials. Therefore using~\pref{lem:coupon}, if we take $mk \ge 2\frac{ (1 - \rho(c))H_{N-1}}{\min_{c' \ne c}\rho(c')}$ boxes, we should cover all coupons with probability at least $1/2$. This implies taking 
$m = \left\lceil\frac{ 2(1 - \rho(c))H_{N-1}}{k\min_{c' \ne c}\rho(c')}\right\rceil$
, we get that the probability of hitting all elements in $\Ccal \setminus\{c\}$ is at least $1/2$, which yields the desired result.
\end{proof}

Substituting back, for the first term, we get,
\begin{align*}
    &\underset{(x,c) \sim D}{\EE}\left[(1 - \tau_k(c))\underset{c^-_i \sim \rho^k}{\EE} \left[ \ell\left(\left\{f(x)^\top \left(\mu_{c} - \mu_{c^-_i} \right)\right\}\right) \middle| I = \phi, c \right]\right]\\
    &\ge \underset{(x,c) \sim D}{\EE}\left[\frac{1 - \tau_k(c)}{2\left\lceil\frac{ 2(1 - \rho(c))H_{N-1}}{k\min_{c' \ne c}\rho(c')}\right\rceil}\cdot \ell\left(\{f(x)^\top \left(\mu_{c} - \mu_{c^\prime} \right)\}_{c^\prime \in \Ccal \setminus c}\right)\right]\\
    &\ge \frac{\left(1 - \rho_{\max}\right)^k}{2\left\lceil\frac{ 2(1 - \rho_{\min})H_{N-1}}{k\rho_{\min}}\right\rceil}\underset{(x,c) \sim D}{\EE}\left[\ell\left(\{f(x)^\top \left(\mu_{c} - \mu_{c^\prime} \right)\}_{c^\prime \in \Ccal \setminus c}\right)\right]\\
    &\ge \frac{\left(1 - \rho_{\max}\right)^k}{2\left\lceil\frac{ 2(1 - \rho_{\min})H_{N-1}}{k\rho_{\min}}\right\rceil}\Lcal^\mu_{\super}(f).
\end{align*}
This gives us the desired result.
\end{proof}

\subsection{Refining the Transfer Theorem}\label{app:split}
As \citet{saunshi2019theoretical} point out, we cannot drive $\Lcal_{\unsup}$ to 0 as it is always lower bounded by $\tau_k$. To get a better understanding of when the loss is actually small, let us define two terms that will be useful for a more refined analysis.
\begin{definition}[Average Intra-class Variance] \label{def:sf}We define the average intra-class variance of a learned embedding as,
\begin{align*}
s(f) \defeq \underset{c \sim \rho}{\EE} \left[\underset{x\sim D_{c}}{\EE}\left[\|f(x)\|\right]\sqrt{\|\Sigma(f, c)\|_2}\right]
\end{align*}
where $\Sigma(f, c)= \EE_{x \sim D_c}[(f(x) - \mu_c)(f(x) - \mu(c))^\top]$ is the covariance matrix of the representation $f$ restricted to class $\Dcal_c$.
\end{definition}

\begin{definition}[NCE loss without collisions] We define a loss function that is obtained from the NCE loss by removing the colliding negatives,
\[
\Lcal^{(k)}_{\ne}(f) \defeq \underset{\substack{c, c_i^{-} \sim \rho^{k+1}\\x, x^{+} \sim D_{c}^2\\x^{-}_i \sim D_{c^{-}_i}}}{\EE}\left[ \ell\left(\{f(x)^\top \left(f(x^{+}) - f(x^{-}_i\right)\}_{i \not\in I}\right)\right]
\]
\end{definition}
Observe that both $\Lcal^{(k)}_{\ne}$ as well as $s(f)$ can be made very small for large enough $\Fcal$. Now we are ready to present our refined bound.
\begin{theorem}[General Refined Bound]
For all $f \in \Fcal$, we have
\begin{align*}
    \Lcal^{\super}(f) \le \frac{2\left\lceil\frac{ 2(1 - \rho_{\min})H_{N-1}}{k\rho_{\min}}\right\rceil}{(1 - \rho_{\max})^k} \left( \Lcal^{(k)}_{\ne}(f) + \sqrt{k\rho_{\max}} \cdot s(f)\right).
\end{align*}
\end{theorem}
\begin{proof}
Decomposing the NCE loss gives us,
\begin{align*}
    &\Lcal^{(k)}_{\unsup}(f) - \tau_k \underset{c, c^-_i \sim \rho^{k+1}}{\EE} \left[\ell_{|I|}(0) \middle| I \ne \emptyset\right]\\
    &\le \underset{\substack{c, c_i^{-} \sim \rho^{k+1}\\x, x^{+} \sim D_{c}^2\\x^{-}_i \sim D_{c^{-}_i}}}{\EE}\left[ \ell\left(\{f(x)^\top \left(f(x^{+}) - f(x^{-}_i\right)\}_{i \not\in I}\right) + \ell\left(\{f(x)^\top \left(f(x^{+}) - f(x^{-}_i\right)\}_{i \in I}\right) \right] - \tau_k \underset{c, c^-_i \sim \rho^{k+1}}{\EE} \left[\ell_{|I|}(0) \middle| I \ne \emptyset\right]\\
    &= \underset{\substack{c, c_i^{-} \sim \rho^{k+1}\\x, x^{+} \sim D_{c}^2\\x^{-}_i \sim D_{c^{-}_i}}}{\EE}\left[ \ell\left(\{f(x)^\top \left(f(x^{+}) - f(x^{-}_i\right)\}_{i \not\in I}\right)\right] + \tau_k \underbrace{\underset{\substack{c, c_i^{-} \sim \rho^{k+1}\\x, x^{+} \sim D_{c}^2\\x^{-}_i \sim D_{c^{-}_i}}}{\EE}\left[ \ell(\{f(x)^\top(f(x^+) - f(x_i^-))\}_{i \in I}) - \ell_{|I|}(0)\middle | I\ne \phi\right]}_{\Delta(f)}\\
    &= \Lcal^{(k)}_{\ne} + \tau_k \Delta(f).
\end{align*}
As before, note that the second term samples the classes $(c,c_i^{-})$ from the conditional distribution given that there is a collision. We follow the convention that an empty set implies value 0 for $\Lcal^{(k)}_{\ne}$. Note that this term can be made very small if the classes are separable by $f$ with a sufficient margin. As for the second term, we will first prove the following lemma,
\begin{lemma}\label{lem:sf}
For any $ t\geq 0$:
\begin{align}
\underset{x, x^+, x_i \sim D_{c}^{2 + t}}{\EE}\left[\ell\left(\{f(x)^\top \left(f(x^{+}) - f(x_i\right)\}_{i \in [t]}\right) -\ell_{t}(0) \right] \le \sqrt{t} \EE_{x\sim D_{c}}[\|f(x)\|^2]\sqrt{\|\Sigma(f, c)\|_2}.
\end{align}
\end{lemma}
The above lemma is an improvement to Lemma A.1 of~\citet{saunshi2019theoretical} which has a linear scaling with $t$. Observe that on the LHS, all samples are drawn from $D_c$, which corresponds to the situation where we have $t$ collisions with the anchor class $c$.
\begin{proof}[Proof of~\pref{lem:sf}]
From~\pref{lem:losses}, we know that 
\[
\underset{x, x^+, x^-_i \sim D_{c}^{2 + t}}{\EE}\left[\ell\left(\{f(x)^\top \left(f(x^+) - f(x^{-}_i)\right)\}_{i \in [t]}\right) -\ell_{t}(0) \right] \le \underset{x, x^+, x^-_i \sim D_{c}^{2 + t}}{\EE}\left[\left|\max_i f(x)^\top \left(f(x^-_i) - f(x^{+})\right)\right|\right].
\]
Now we have,
\begin{align*}
    &\underset{x, x^+, x^-_i \sim D_{c}^{2 + t}}{\EE}\left[\left|\max_i f(x)^\top \left( f(x^-_i) - f(x^{+})\right)\right|\right] \\
    &= \underset{x\sim D_{c}}{\EE}\left[\|f(x)\|\underset{x^+, x^-_i \sim D_{c}^{1 + t}}{\EE}\left[\left|\max_i \frac{f(x)^\top}{\|f(x)\|} \left(f(x^-_i) - f(x^{+})\right)\right|\right]\right]\\
    &\le \underset{x\sim D_{c}}{\EE}\left[\|f(x)\|\sqrt{\underset{x^+, x_i \sim D_{c}^{1 + t}}{\EE}\left[\left(\max_i \frac{f(x)^\top}{\|f(x)\|} \left(f(x^-_i) - f(x^{+})\right)\right)^2\right]}\right] \\
    &\le \underset{x\sim D_{c}}{\EE}\left[\|f(x)\|\sqrt{\sum_i\underset{x^+, x^-_i \sim D_{c}^{2}}{\EE}\left[\left( \frac{f(x)^\top}{\|f(x)\|} \left(f(x_i^-) - f(x^{+})\right)\right)^2\right]}\right] \\
    &\le \underset{x\sim D_{c}}{\EE}\left[\|f(x)\|\sqrt{t\underset{x^+, x^- \sim D_{c}^{2}}{\EE}\left[\left( \frac{f(x)^\top}{\|f(x)\|} \left(f(x^-) - f(x^{+})\right)\right)^2\right]}\right] \\
    &\le \sqrt{t} \underset{x\sim D_{c}}{\EE}\left[\|f(x)\|\right]\sqrt{\|\Sigma(f, c)\|_2}.
\end{align*}
 Combining with the above inequality gives us the desired result. 
\end{proof}

Substituting back, we get the following bound on the second term,
\begin{align*}
    \tau_k\Delta(f) &\leq  \underset{c \sim \rho}{\EE} \left[\tau_k(c)\underset{x\sim D_{c}}{\EE}\left[\|f(x)\|\right]\sqrt{\|\Sigma(f, c)\|_2}\underset{c^{-}_i \sim \rho^k}{\EE}\left[\sqrt{|I|} \middle| I \ne \emptyset, c \right]\right]\\
    &\le \underset{c \sim \rho}{\EE} \left[\tau_k(c)\underset{x\sim D_{c}}{\EE}\left[\|f(x)\|\right]\sqrt{\|\Sigma(f, c)\|_2}\sqrt{\underset{c^{-}_i \sim \rho^k}{\EE}\left[|I| \middle| I \ne \emptyset, c \right]}\right]\\
    &\labelrel={in:expI} \underset{c \sim \rho}{\EE} \left[\underset{x\sim D_{c}}{\EE}\left[\|f(x)\|\right]\sqrt{\|\Sigma(f, c)\|_2}\sqrt{k \rho(c)\tau_k(c)}\right]\\
    &= \sqrt{k}\underset{c \sim \rho}{\EE} \left[\underset{x\sim D_{c}}{\EE}\left[\|f(x)\|\right]\sqrt{\|\Sigma(f, c)\|_2}\sqrt{\rho(c)\tau_k(c)}\right]\\
    &\labelrel\le{in:applysf} \sqrt{k \rho_{\max}}s(f)
\end{align*}
where \eqref{in:expI} follows from observing that $\underset{c^{-}_i \sim \rho^k}{\EE}\left[|I| \middle| I \ne \emptyset, c \right] = \frac{k \rho(c)}{\tau_k(c)}$ (see also equation (35) of \citet{saunshi2019theoretical}) and \eqref{in:applysf} follows from~\pref{def:sf}. Combining these terms gives us the desired result.
\end{proof}
Note that~\citet{saunshi2019theoretical} also obtain a similar result, but in their analysis the coefficient on the $s(f)$ term scales linearly with $k$. Instead, we obtain a $\sqrt{k}$ scaling. However, for the case of uniform class distribution $\rho$, their coefficient is tighter in the $k \leq n$ regime.

\subsection{Generalization Bound}
\label{app:gen}
To transfer our guarantee to the minimizer of the empirical NCE loss, we prove a generalization bound using a notion of worst case Rademacher complexity.
\begin{definition}[Rademacher Complexity]\label{def:rademacher}
The empirical Rademacher complexity of a real-valued function class $\Hcal$ for a set $S$ drawn from $\Xcal^m$ is defined to be
\[
\hat{\Rcal}_S(\Hcal) \defeq \EE_\sigma\left[\sup_{h \in \Hcal}\left(\frac{1}{m} \sum_{i=1}^m \sigma_ih(x_i)\right)\right],
\]
where $\sigma_i$ are iid Rademacher random variables.
We define the worst-case Rademacher complexity to be
\[
\Rcal_m(\Hcal) \defeq \max_{S \sim \Xcal^m} \hat{\Rcal}_S(\Hcal).
\]
\end{definition}

We will use the following theorem to bound the Rademacher complexity of our vector valued class,
\begin{theorem}[\cite{foster2019ell_}]\label{thm:composition}
Let $\Hcal \subseteq \{h: \Xcal \rightarrow \mathbb{R}^K\}$, and let $\phi$ be a real-valued $L$-Lipschitz function with respect to the $\ell_\infty$ norm. We have,
\[
\Rcal_S(\phi \circ \Hcal) \le \tilde{O}\left(L \sqrt{k}\right) \max_i \Rcal_M(\Hcal|_i)
\]
where $\Hcal|_i$ is the function class restricted to the $i^{\text{th}}$ coordinate.
\end{theorem}
Note that the above theorem relates the empirical Rademacher complexity on sample $S$ to the worst case Rademacher complexity. Finally we will use the following standard generalization result based on Rademacher complexity,
\begin{theorem}[\cite{mohri2018foundations}]\label{thm:generalization}
For a real-valued function class $\Hcal$ and a set $S = \{x_1, \ldots, x_m\}$ $\in \Xcal^m$ drawn i.i.d. from some distribution, with probability $1- \delta$, for all $h \in \Hcal$
\[
\abr{ \EE[h(x)] - \frac{1}{m}\sum_{i=1}^m h(x_i)} \leq \Rcal_S(\Hcal) + O\left(\sqrt{\frac{\log(1/\delta)}{m}}\right).
\]
\end{theorem}

Now we state our generalization bound for the NCE loss.
\begin{theorem}
Let $B = \max_{x \in \Xcal, f' \in \Fcal} \|f'(x)\|_1$. For all $f \in \mathcal{F}$, with probability $1-\delta$ over the random sample $S$ of $M$ NCE examples used to fit $\hat{f}$,
\begin{align*}
\Lcal^{(k)}_{\unsup}(\hat{f}) \leq \Lcal^{(k)}_{\unsup}(f) +  \tilde{O}\left(B\sqrt{(k+2)d}\right) \max_i \Rcal_M(\Fcal|_i) + O\left(\sqrt{\frac{\log(1/\delta)}{M}}\right).
\end{align*}
\end{theorem}
\begin{proof}
Let us assume the dimension of input features is $n$. Define function $\phi: \mathbb{R}^{d(k+2)} \rightarrow \mathbb{R}$ as
\[
\phi(z, z^+, z^-_1, \ldots, z^-_{k}) = \ell\left(\left\{z^T(z^+ - z^-_i) \right\}_{i=1}^k\right)
\]
where each $z, z^+, z^-_i \in \mathbb{R}^d$. Also define the class of functions $\Gcal \defeq \{g_f| f \in \Fcal\}$ where $g_f: \mathbb{R}^{n(k+2)} \rightarrow \mathbb{R}^{d(k+2)}$ such that
\[
g_f(x, x^+, x^-_1, \ldots, x^-_{k}) = (f(x), f(x^+), f(x^-_1), \ldots, f(x^-_k)).
\]
Observe that $\phi \circ g_f$ gives us exactly the NCE loss for a sample. Therefore using~\pref{thm:generalization} and the standard ERM analysis, we have with probability $1- \delta$ for all $f \in \Fcal$,
\[
\Lcal^{(k)}_{\unsup}(\hat{f}) \le  \Lcal^{(k)}_{\unsup}(f) + 2\Rcal_S(\phi \circ \Gcal) + O\left(\sqrt{\frac{\log(1/\delta)}{M}}\right).
\]
Now we will show that $\phi$ is Lipschitz. Observe that, for any $W = (w, w^+, w^-_1, \ldots, w^-_{k}), Z = (z, z^+, z^-_1, \ldots, z^-_{k})$,
\begin{align*}
   |\phi(Z) - \phi(W)| &= |\phi(z, z^+, z^-_1, \ldots, z^-_{k}) - \phi(w, w^+, w^-_1, \ldots, w^-_{k})|\\
    &= \left| \ell\left(\left\{z^T(z^+ - z^-_i) \right\}_{i=1}^k\right) - \ell\left(\left\{w^T(w^+ - w^-_i) \right\}_{i=1}^k\right)\right|\\
    &\le \left|\max_i (w^T(w^+ - w^-_i) - z^T(z^+ - z^-_i))\right|\tag{Using~\pref{lem:losses}}\\
    &\le \left| w^\top w^+ - z^\top z^+ + \max_i z^\top z_i^- - w^\top w_i^- \right|\\
    &\leq \left|w^\top w^+ - z^\top z^+\right| + \max_i \left|w^\top w_i^+ - z^\top z_i^+\right|
\end{align*}
Thus we have several terms with a similar form $|w^\top \tilde{w} - z^\top \tilde{z}|$ where $\tilde{w} \in W, \tilde{z} \in Z$. For these, we bound as
\begin{align*}
|w^\top \tilde{w} - z^\top \tilde{z}| \leq |w^\top(\tilde{w} - \tilde{z}) + (w-z)^\top\tilde{z}| \leq 2B \|W - Z\|_{\infty},
\end{align*}
where $\|\cdot\|_{\infty}$ is the $\ell_{\infty}$ norm in the
$d(k+2)$ dimensional space. Thus we conclude that 
\begin{align*}
|\phi(Z) - \phi(W)| \leq 4 B \|W - Z\|_{\infty}.
\end{align*}
Now applying Theorem \ref{thm:composition}, we have
\[
\Rcal_S(\phi \circ \Gcal) \le \tilde{O}\left(B \sqrt{(k+2)d}\right) \max_{i\in[d(k+2)]} \Rcal_M(\Gcal|_i) \le \tilde{O}\left(B\sqrt{(k+2)d}\right) \max_{i \in [d]} \Rcal_M(\Fcal|_i).
\]
Substituting back gives us the desired result. Here, to pass from the Rademacher complexity of $\Gcal$, which has output dimension $d(k+2)$, we use the fact that $g_f$ is simply $k+2$ copies of $f$ applied to different examples and that we are working with the worst-case Rademacher complexity. 
\end{proof}
Applying this for $\hat{f}$ and using the fact that it is the minimizer of the empirical NCE loss gives us the requires guarantee on the NCE loss of $\hat{f}$.
\section{Calculations for the Example in~\pref{sec:theory}}
\label{app:calculations-section-4.1}

\paragraph{Properties of $f_1$.}
First, it is easy to see that $\Lcal_{\sup}^{\mu}(f_1) = 1$ for the hinge loss, since the mean embeddings satisfy $\mu_{2i} = \mu_{2i+1} = e_i$. In addition, every example is embedded as $e_j$ for some $j$, which means that $W^\mu e_j$ is a $2$-sparse binary vector. Such score vectors always yield hinge loss $1$.

Now, let us verify the stated identify for $\Lcal_{\unsup}^{(k)}(f_1)$. Due to symmetry, we can assume that $x,x^+$ come from class $1$. Now, using the definition of $f_1$ we see that
\begin{align*}
    \max\{0, 1 + f_1(x)^\top(f_1(x_i^-) - f_1(x^+))\} = \one\{c_i^{-} \in \{1,2\}\}.
\end{align*}
Taking max over all of the negatives, we see that the NCE loss will be $1$ if and only if class $1$ or $2$ appear in the negatives and otherwise it will be $0$. Formally
\begin{align*}
    \Lcal_{\unsup}^{(k)}(f_1) = \PP_{S \sim \rho^k}[ S \cap \{1,2\} \ne \emptyset ] = 1 - \PP_{S \sim \rho^k}[S \cap \{1,2\} = \emptyset] = 1 - (1-2/N)^k,
\end{align*}
where in the last step we use that the prior $\rho$ is uniform. Now the stated identity follows by algebraic manipulations. In particular:
\begin{align*}
    \Lcal_{\unsup}^{(k)}(f_1) - \tau_k & = 1 - (1-2/N)^k - 1 + (1-1/N)^k = \left(\frac{N-1}{N}\right)^k \left(1 - \left(1 - \frac{1}{N-1}\right)^k\right) \\
    & = (1-\tau_k)\left( 1- \left(1 - \frac{1}{N-1}\right)^k\right)
\end{align*}

Finally, observe that
\begin{align*}
    \lim_{k \to \infty} \Lcal_{\unsup}^{(k)}(f_1) = 1 - \lim_{k \to \infty}(1 - 2/N)^k = 1.
\end{align*}
Note that for $f_1$,
\[
\Lcal^\mu_{\super} = 1 = \frac{1}{(1-\tau_k)\left( 1- \left(1 - \frac{1}{N-1}\right)^k\right)}\left(\Lcal_{\unsup}^{(k)}(f_1) - \tau_k\right) \ge \frac{N-1}{(1-\tau_k)k}\left(\Lcal_{\unsup}^{(k)}(f_1) - \tau_k\right)
\]
where the inequality follows from observing $1 - (1-x)^r \le rx$ for $0 <x < 1$. This shows that~\pref{thm:transfer} is tight up to $\log$ factors.

\paragraph{Properties of $f_2$.}
For $f_2$, the supervised loss is computed as follows. First note that $\mu_i = e_i$ for each class. Now, observe that if the example $x$ is from class $i$ we have that $f_2(x)^\top \mu_j = 0$ for all $j \ne i$. For class $i$ we have $f_2(x)^\top \mu_i \sim \textrm{Unif}(\{1-\epsilon,1+\epsilon\})$ where the randomness is over the realization of $x$. In the first of these cases, the hinge loss is $\epsilon$, while in the second case it is $0$. Thus, we have $\Lcal_{\sup}^\mu(f_2) = \epsilon/2$. 

To analyze $\Lcal_{\unsup}$ we consider two cases. Again, due to symmetry we can assume that $x,x^+$ come from class $1$. 
\paragraph{Case 1, no collisions:} If none of the negative examples come from class $1$ then $f_2(x)^\top f_2(x_i') = 0$ for all $i$, hence the hinge loss is (where $\sigma,\sigma^+$ are Rademacher random variables):
    \begin{align*}
        \EE_{\sigma,\sigma^+} \max\{0, 1 + (1+\sigma\epsilon) (-1-\sigma^+\epsilon)\} = \epsilon^2/4 + \epsilon/2.
    \end{align*}
    (With probability $3/4$ over the Rademacher variables, the hinge structure clips the loss at $0$.)
\paragraph{Case 2, collisions:} We provide crude upper and lower bounds in the case of collisions. Assume that $x_i^{-}$ is a collision, so that this term in the hinge loss is:
    \begin{align*}
        & \EE_{\sigma,\sigma^+,\sigma_i^-} \max\{0, 1 + (1+\sigma\epsilon)(1 + \sigma_i^-\epsilon - 1 - \sigma^+\epsilon)\} = \EE_{\sigma,\sigma^+,\sigma_i^-} \max\{0, 1 + (1+\sigma\epsilon)((\sigma_i^- - \sigma^+)\epsilon)\} \\
        & \geq  \max\{0, \EE_{\sigma,\sigma^+,\sigma_i^-} 1 + (1+\sigma\epsilon)((\sigma_i^- - \sigma^+)\epsilon)\} = 1.
    \end{align*}
    Here, the inequality is Jensen's inequality and then we use independence of the Rademacher random variables. This is clearly a lower bound on the expected NCE loss when there is a collision.
    We can also upper bound the loss when there is a collision by setting $\sigma_i^+ = 1$
    \begin{align*}
        \EE_{\sigma,\sigma^+,\sigma_i^-}\max\{0, 1+(1+\sigma\epsilon)((\sigma_i^- - \sigma^+)\epsilon\} \leq \EE_{\sigma,\sigma^+}\max\{0, 1 + (1+\sigma\epsilon)(1-\sigma^+)\epsilon\} = 1+\epsilon.
    \end{align*}
    In addition, since $\epsilon \leq 1$, for any values of $\sigma,\sigma^+$, this term dominates the loss incurred from any other negative. Thus this is an upper bound on the NCE-loss when there is a collision. 
Thus, we have the bounds
\begin{align*}
    (1-\tau_k)(\epsilon^2/4 + \epsilon/2) + \tau_k \leq  \Lcal_{\unsup}^{(k)}(f_2) \leq (1-\tau_k)(\epsilon^2/4 + \epsilon/2) + \tau_k (1+\epsilon).
\end{align*}

Actually we can obtain a much better lower bound. Note that the smallest realization of a collision loss is $1 + (1+\epsilon)(-2\epsilon) = 1 - 2\epsilon - 2\epsilon^2$. This quadratic is positive for $\epsilon \in [0, (\sqrt{3}-1)/2]$, which we assume going forward. With this choice of $\epsilon$, there will be no clipping and we can ignore taking the max with $0$. Next, observe that the loss we incur for a non-colliding example is $\max\{0, 1+(1+\sigma\epsilon)(-1 - \sigma^+\epsilon)\}$. We can see that for any realizations of $\sigma,\sigma^+,\sigma_i^-$ 
\begin{align*}
    1+(1+\sigma\epsilon)(-1 - \sigma^+\epsilon) \leq 1 + (1+\sigma\epsilon)(\sigma_i^- - \sigma^+)\epsilon, 
\end{align*}
which means that we can ignore all of the non-colliding examples when evaluating the loss. Thus, the loss is
\begin{align*}
    \EE \max\{0, 1+(1+\sigma \epsilon)(\max_{i \in I} \sigma_i^- - \sigma^+)\epsilon \} = 1 + \epsilon \cdot \EE \max_{i \in I} \sigma_i^-, %
\end{align*}
where $I$ is the set of colliding examples. This is clearly upper bounded by $\tau_k(1+\epsilon)$ after taking into account the probability of a collision. For a lower bound, we have
\begin{align*}
    &\EE\sbr{\one\{|I|>0\}(1 + \epsilon \max_{i \in I} \sigma_i)} = \tau_k + \epsilon \EE\sbr{\one\{|I| > 0\}\max_{i \in I}\sigma_i}\\
    & = \tau_k + \epsilon\left(\sum_{t=1}^k {k \choose t} \rbr{\frac{1}{N}}^t \rbr{1 - \frac{1}{N}}^{k-t} \rbr{ (1 - (1/2)^t)\cdot 1 + (1/2)^t\cdot(-1)}\right)\\
    & = \tau_k + \epsilon\left(\sum_{t=1}^k {k \choose t} \rbr{\frac{1}{N}}^t \rbr{1 - \frac{1}{N}}^{k-t} \rbr{ 1 - 2 (1/2)^t}\right)\\
    & = \tau_k + \epsilon\left(\tau_k - 2 \sum_{t=1}^k {k \choose t} \rbr{\frac{1}{N}}^t \rbr{1 - \frac{1}{N}}^{k-t} (1/2)^t\right)\\
    & \geq \tau_k + \epsilon\left(\tau_k - \frac{k}{N}\rbr{1 - \frac{1}{N}}^{k-1} - \frac{1}{2} \sum_{t=2}^k {k \choose t} \rbr{\frac{1}{N}}^t \rbr{1 - \frac{1}{N}}^{k-t}\right)\\
    & = \tau_k + \epsilon/2\cdot\rbr{\tau_k - \frac{k}{N}\rbr{1 - \frac{1}{N}}^{k-1}}.
\end{align*}
Here the only lower bound uses that in the terms where $t \geq 2$ we can upper bound $(1/2)^t \leq 1/4$. Combining this with the term from the ``no-collision'' case, we have
\begin{align*}
    \Lcal_{\unsup}^{(k)}(f_2) & \geq (1-\tau_k)(\epsilon^2/4 + \epsilon/2) + \tau_k + \epsilon/2\rbr{\tau_k - \frac{k}{n}\rbr{1 - \frac{1}{N}}^{k-1}}\\
    & = \tau_k  + (1-\tau_k)\epsilon^2/4 + \epsilon/2\rbr{1 - \frac{k}{N}\rbr{1-\frac{1}{N}}^{k-1}}.
\end{align*}

Finally, notice that the upper bound on the collision case is tight if there is a collision that has $\sigma_i^+ = 1$. In the asymptotic regime where $k \to \infty$ we have that the probability of such a collision goes to $1$, and simultaneously $\tau_k \to 1$. Thus in this limit this upper bound is tight, so we obtain $\lim_{k \to \infty}\Lcal_{\unsup}^{(k)}(f_2) = 1+\epsilon$.

\paragraph{Dependence on $k/N$.}
Next, we claim that (bounds on) the first crossing point between the NCE losses for $f_1$ and $f_2$ depend on $k$ only through the function $k/N$. To see why, recall the elementary inequalities
\begin{align*}
    1-x \leq (1-x/n)^n \leq \exp(-x) \leq 1 - x + x^2
\end{align*}
This means that a sufficient condition for $f_1$ to be selected by NCE is
\begin{align*}
    \frac{2k}{N} \leq (1-k/N) (\epsilon^2/4 + \epsilon/2) + (k/N - k^2/N^2)
\end{align*}
While a sufficient condition for $f_2$ to be selected is
\begin{align*}
    2k/N - 4k^2/N^2 \geq (1 - k/N + k^2/N^2)(\epsilon^2/4 + \epsilon/2) + k/N (1+\epsilon)
\end{align*}
These are both quadratic inequalities in $k/N$. Thus if they admit solutions, these would have $k$ scaling linearly with $N$. 

\section{Experiment Details}
\label{app:experiment-details}

\subsection{Wiki-3029 NLP Experiments} We present additional details for the Wiki-3029 results below.

\paragraph{Dataset details.} We use the dataset provided by~\citet{saunshi2019theoretical}.\footnote{https://nlp.cs.princeton.edu/CURL/raw/} This dataset was released by the authors, and we also explicitly obtained their permission to use the dataset for this study. We use the scripts provided to us by the authors to generate the Wiki-3029 dataset.

\paragraph{Training details.} We perform NCE training for a fixed number of epochs. In each epoch, we take sentences from each article and pair them randomly to form a set of positively aligned examples $(x, x^+)$. We mix pairs across classes and divide this dataset into batches of a given size $B$ and fit the representation using the Adam optimizer on the NCE loss. Formally, for a given batch $\Bcal = \{(x_i, x_i^+)\}_{i=1}^B$ we optimize the loss:
\begin{equation*}
    \Lcal_{\mathrm{nce}}(\theta) = -\frac{1}{2B} \sum_{i=1}^B \ln\frac{w(x_i, x^+_i)}{\sum_{j \ne i} \left\{w(x_i, x^+_j) + w(x_i, x_j)\right\}} -\frac{1}{2B}\sum_{i=1}^B \ln\frac{w(x_i, x^+_i)}{\sum_{j \ne i} \left\{w(x^+_i, x^+_j) + w(x^+_i, x_j)\right\}},
\end{equation*}
where $w(x, x') = \exp(\nicefrac{f_\theta(x)^\top f_\theta(x')}{\tau})$, $\theta$ is the parameters of the bag-of-words model $f$, and $\tau$ is a temperature hyperparameter. The parameters $\theta$ consists of word embedding for every token in the vocabulary including the {\tt UNK} token.

At the end of NCE training, we compute a mean classifier vector $v_c$ for an article $c$ by averaging $f_\theta(x)$ for every sentence $x$ from article $c$ in the training dataset. For a given sentence $x'$ in the test dataset, we assign it the class $\arg\max_c f_\theta(x')^\top v_c$.

We further train a linear classifier on top of $f_\theta(x)$ for a fixed number of epochs. We optimize the logistic loss and use Adam optimizer. Formally, we use the entire NCE training dataset and do mini-batch learning. Given a batch $\{(x_i, c_i)\}_{i=1}^B$ of pairs $i$ and sentences $c_i$ we optimize the loss:
\begin{equation*}
\Lcal_{\mathrm{downstream}}(W) = - \frac{1}{B}\sum_{i=1}^B \ln \frac{\exp\left( (W^\top f_\theta(x_i))_{c_i}\right)}{\sum_{c'}\exp\left((W^\top f_\theta(x_i))_{c'}\right)},
\end{equation*}
where $W$ is the parameters of the linear classifier. We keep the parameters $\theta$ fixed when training the linear classifier. We save the linear classifier at the end of each epoch and compute the aggregate training loss recorded in this epoch. Finally, we report performance of the model with the smallest aggregate  training loss.

We initialize $W$ and $\theta$ randomly using default PyTorch initialization. 

\paragraph{Hyperparameters.} We used the following hyperparameters for the Wiki-3029 experiments. 

\begin{table*}[h]
    \centering
    \begin{tabular}{l|c}
        \toprule        \textbf{Hyperparameters} & \textbf{Values} \\
        \midrule
        Word Embedding & 768\\
        Optimization method & Adam\\
        Parameter initialization & PyTorch 1.4 default\\
        Temperature & 1\\
        Gradient clip norm & 2.5\\
        NCE epochs & 50\\
        Downstream training epoch & 50\\
        Learning rate & $0.01$\\
        Model architecture & Bag of Words\\
        Data augmentation & None\\
        \bottomrule
    \end{tabular}
    \caption{Hyperparameters for Wiki-3029}
    \label{tab:nlp_hyperparameter}
\end{table*}

\paragraph{Detailed Figures with Error Bars.} 

We show a detailed version of the left plot in~\pref{fig:nlp-exp} in~\pref{fig:wiki3029_n_k_mean_classifier}. This shows the mean classifier accuracy on varying $N$ and $k$ for a fixed value of $B=1000$. The gap between the best and worst performance in each row is quite noticeable ($\approx$ 2-3\%) and this is much larger than the standard deviation ($\le 0.25$).
We also plot the performance of the trained linear classifier in~\pref{fig:wiki3029_n_k_trained_classifier}. Note that the general trend seems to be similar to~\pref{fig:wiki3029_n_k_mean_classifier} and an intermediate value of $k$ seems to be best for any value of $N$. However, we see that in this case the value of $k=50$ is optimal across all values of $N$.

\begin{figure}%
    \centering
    \includegraphics[width=12cm]{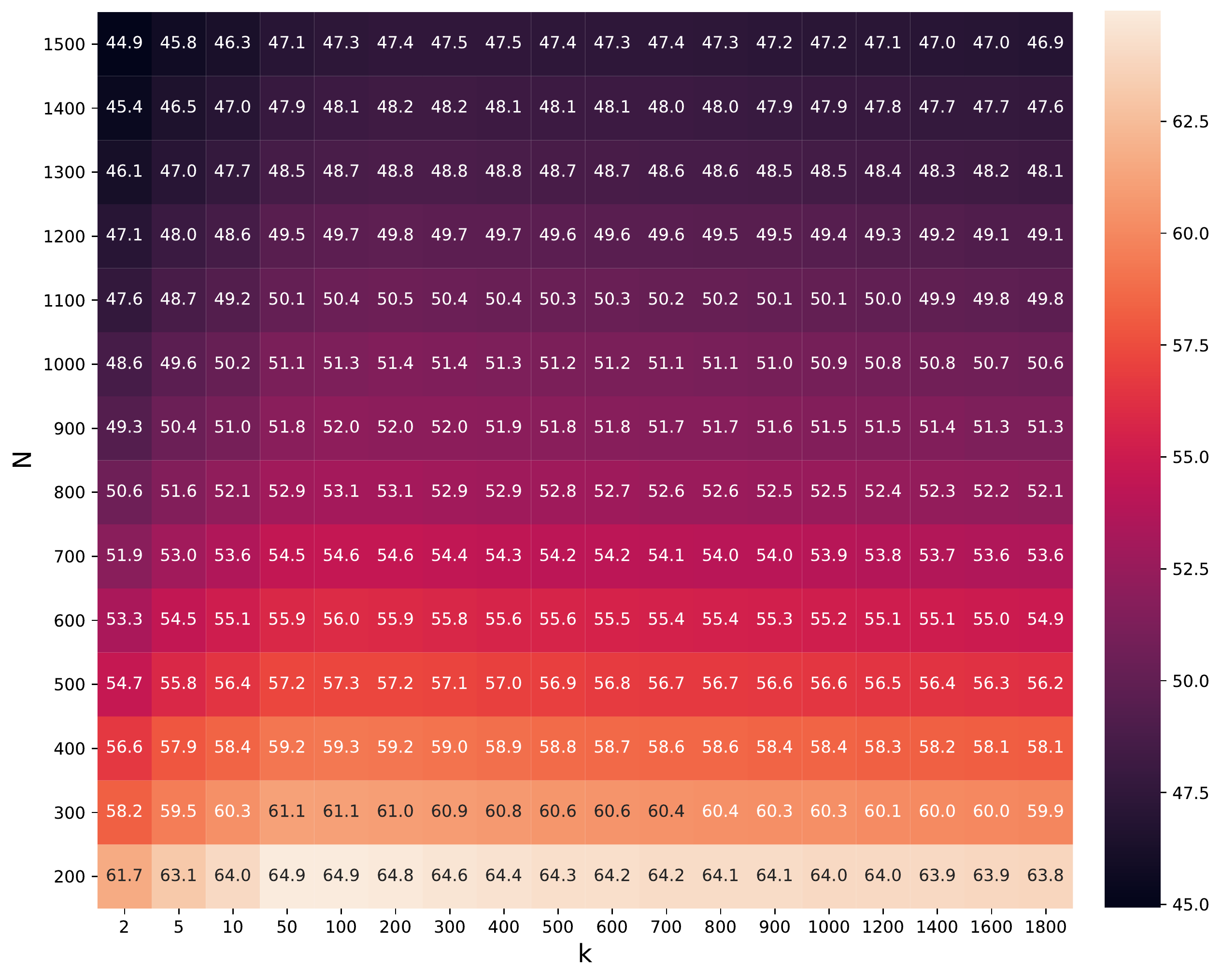}
    \caption{
      Mean classifier accuracy on the Wiki-3029 dataset for different values of number of classes $(N)$ and number of negative examples $(k)$, and a fixed batch size of 1000. 
      Values show average performance across 5 trials with different seeds, and all standard deviations are at most $0.25$ and typically much smaller.
      }
    \label{fig:wiki3029_n_k_mean_classifier}%
\end{figure}

\begin{figure}%
    \centering
    \includegraphics[width=12cm]{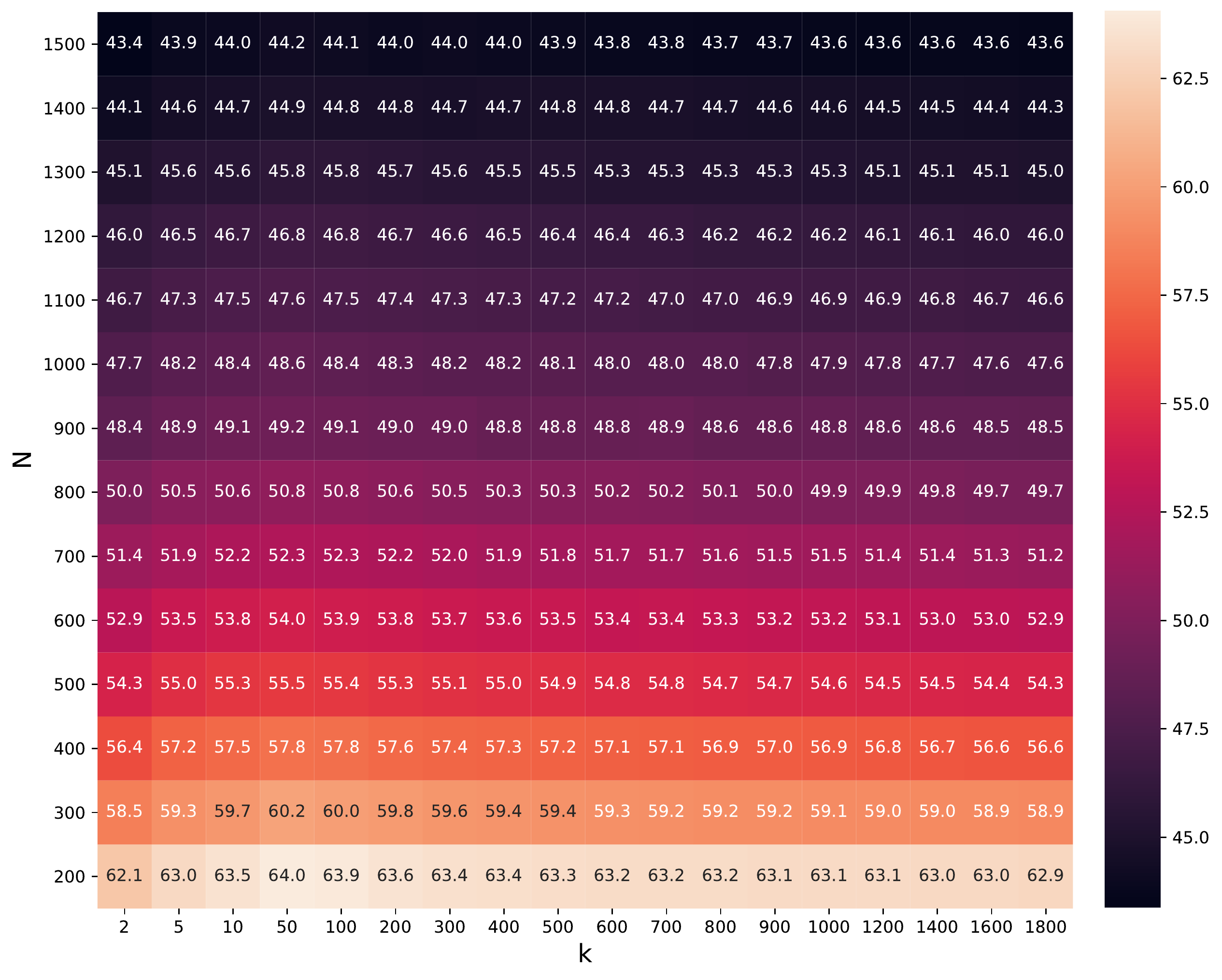}
    \caption{Trained linear classifier accuracy on the Wiki-3029 dataset for different values of number of classes $(N)$ and number of negative examples $(k)$, and a fixed batch size of 1000.
      Values show average performance across 5 trials with different seeds and all standard deviations are at most $0.28$ and typically much smaller.}
    \label{fig:wiki3029_n_k_trained_classifier}%
\end{figure}

We show a detailed version of the center plot in~\pref{fig:nlp-exp} in~\pref{fig:wiki3029_b_vs_k}. This shows the mean and trained classifier accuracy for different values of $B, k$, and $N$. We observe that gap between the best and worst performance is noticeable ($\approx$1\% for the mean classifier and $\approx$2-3\% for the trained classifier). Further, standard deviation is sufficiently low to enable reading trends from the mean performance. 
For the trained classifier, we observe that optimal performance is received with largest value of $B$ and an intermediate value of $k$ similar to the mean classifier. 
\begin{figure}%
    \centering
    \subfigure[Mean classifier]{{\includegraphics[width=7.5cm]{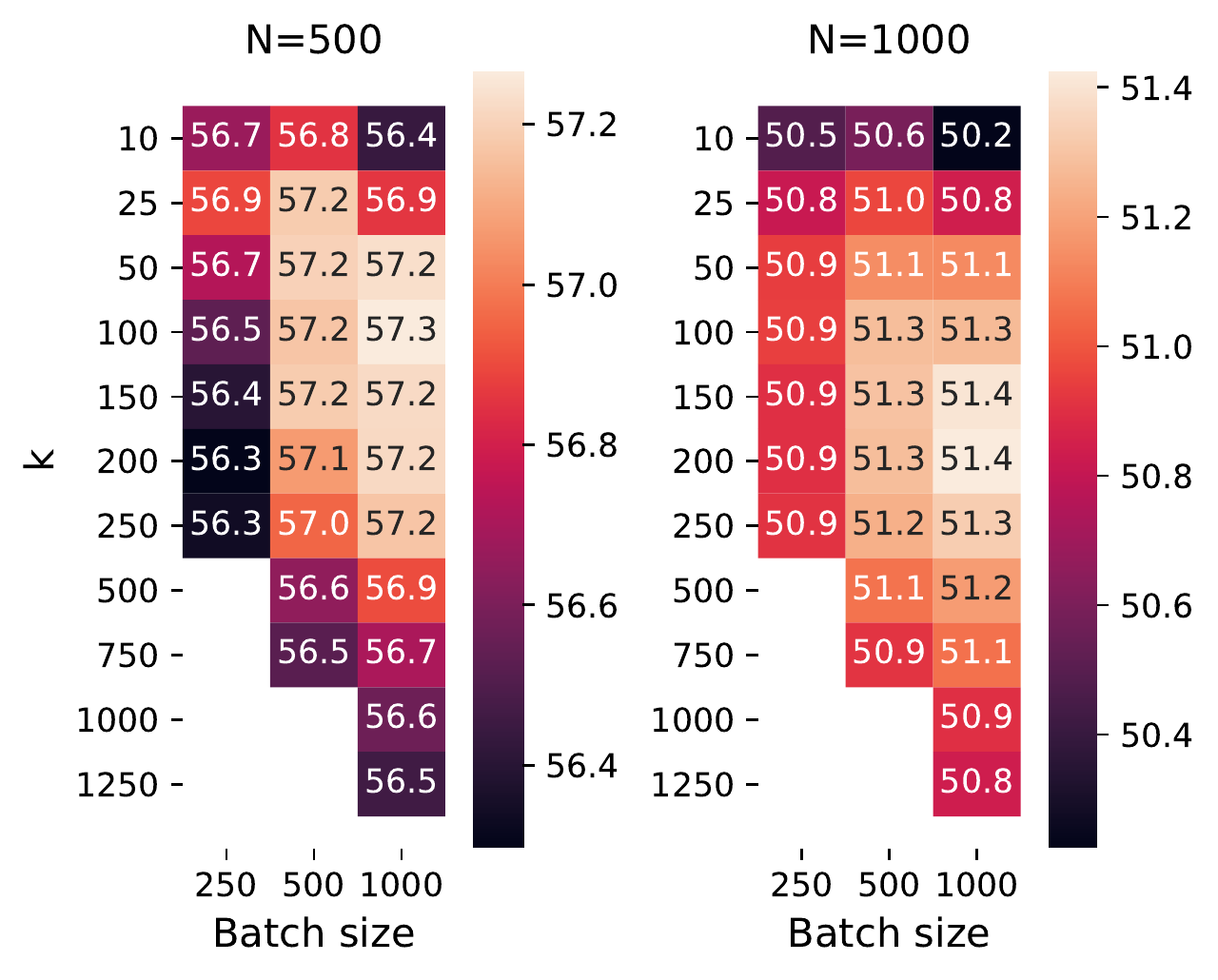} }}%
    \quad
    \subfigure[Trained classifier]{{\includegraphics[width=7.5cm]{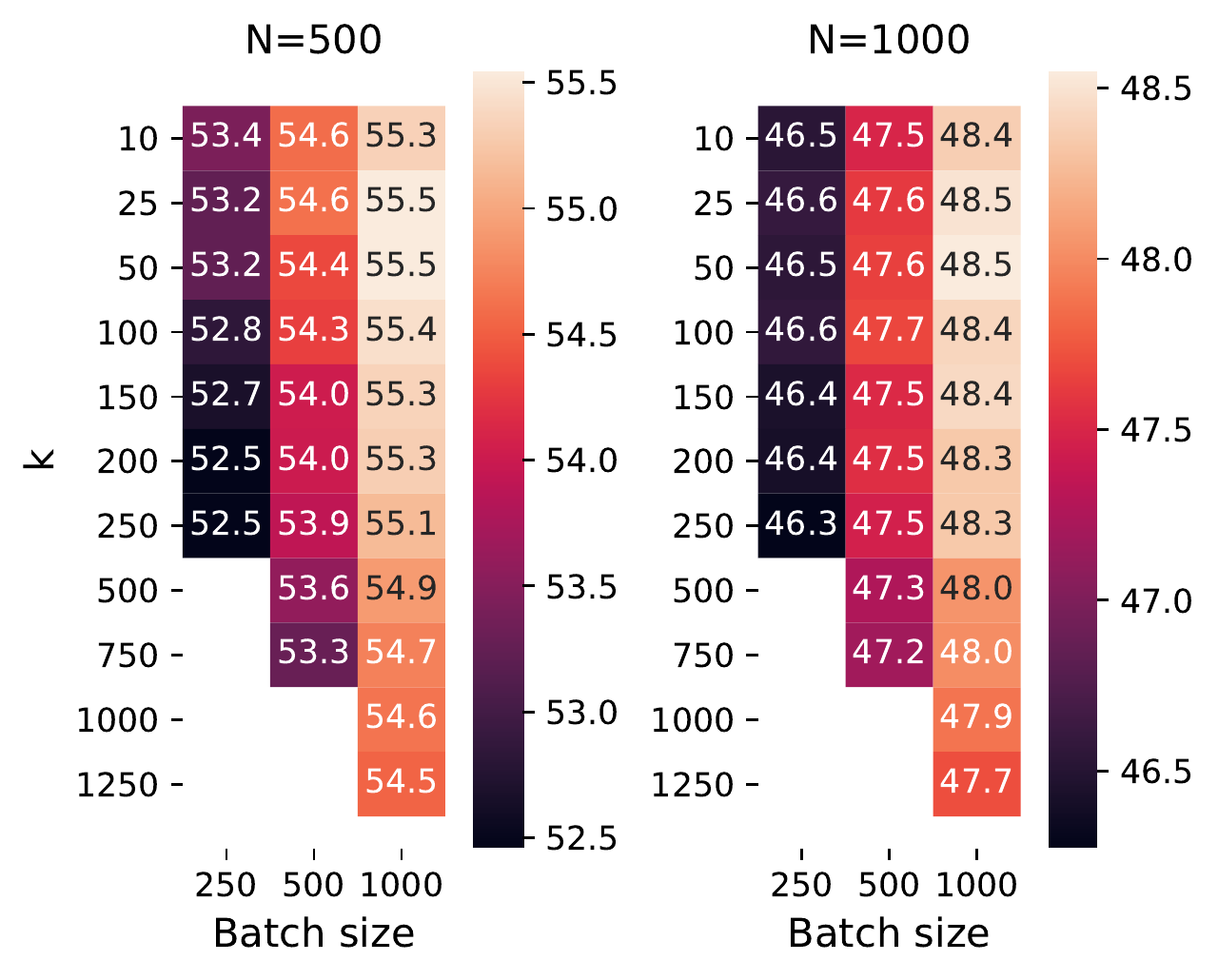} }}%
    \caption{Mean and trained classifier accuracy on Wiki-3029 dataset
      for different values of batch size $B$ and number of negatives
      examples $k$ and two values of the number of classes $N$. \textbf{Left:} 
      mean classifier performance (All standard deviations are less
      that $0.16$). \textbf{Right:} trained linear classifier performance (All
      standard deviations are less that $0.26$).}
    \label{fig:wiki3029_b_vs_k}
\end{figure}

\paragraph{Compute Infrastructure.} We run our experiments on a cluster containing a mixture of P40, P100, and V100 GPUs. Each experiment runs on a single GPU in a docker container and takes less than 5 hours to complete with a V100 GPU.

\subsection{Vision Experiment Details}

Vision experiments used the same infrastructure as NLP experiments, but took roughly two days to fully run. We train models on the NCE task for 400 epochs. We used the following hyperparameters to run the included CIFAR-10 experiments:

\begin{table*}[h]
    \centering
    \begin{tabular}{l|c}
        \toprule        \textbf{Hyperparameters} & \textbf{Values} \\
        \midrule
        Latent dimension & 128\\
        Gradient clip norm & None\\
        Optimization method & Adam\\
        Parameter initialization & PyTorch 1.4 default\\
        Learning rate & $10^{-4}$\\
        Model architecture & ResNet-18\\
        \bottomrule
    \end{tabular}
    \caption{Hyperparameters for vision}
    \label{tab:cifar_hyperparameter}
\end{table*}

\paragraph{Additional Experiments.}
Here we provide a plot similar to \pref{fig:visionExp}, but using a mean classifier rather than a trained, linear classifier for the downstream task. The trend in \pref{fig:visionExpMeans} is similar, with results varying based on both dataset and whether or not augmentation is present. Interestingly, mean classifier performance on SVHN with augmented data is significantly worse than the trained classifier version in \pref{fig:visionExp}. This is possibly due to the fact that the augmentation was not designed for digit data, including, for example, horizontal flips.

\begin{figure}[t]
\hspace{-0.8cm}
\includegraphics[trim={0cm, 0cm, 0cm, 0}, clip, scale=0.6]{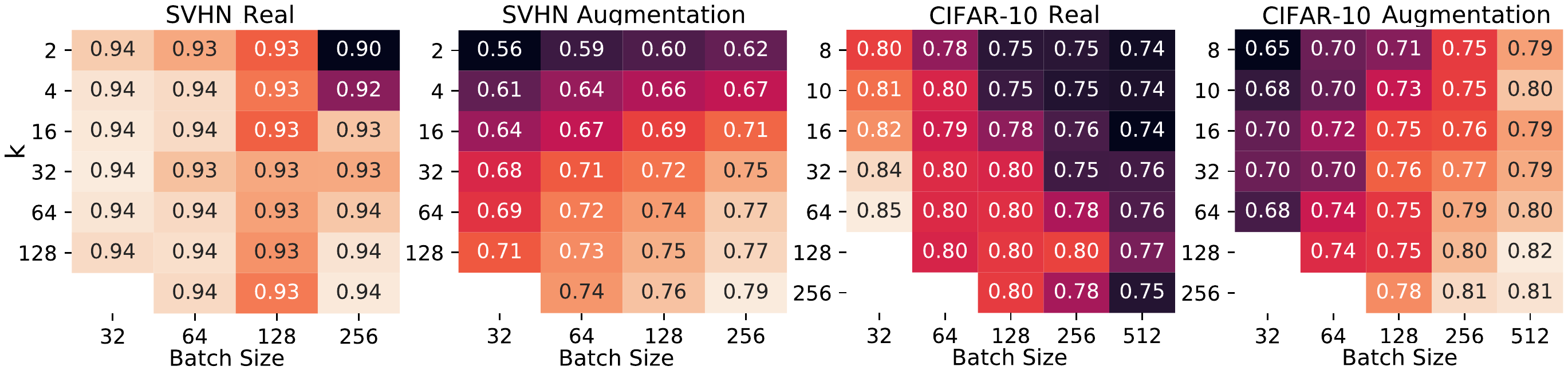}
\caption{A companion to \pref{fig:visionExp}, showing mean classifier performance, as used in our theory, rather than trained classifier performance. Here we show classifier accuracy as a function of both batch size and $k$ for various vision datasets and NCE parameters. \textbf{Left}: SVHN data, using a true positive sample from the same class, as used in our analysis. \textbf{Center Left}: SVHN data, but using data augmentation. Here the positive sample is an augmented version of the reference sample, and negatives are augmented as well. \textbf{Center Right}: NCE on CIFAR-10 data, using real positive examples. \textbf{Right:} CIFAR-10 results with data augmentation.}
\label{fig:visionExpMeans}
\end{figure}

\end{document}